\documentclass{article}
\PassOptionsToPackage{numbers,compress}{natbib}
\usepackage[main, final]{neurips_2026}

\usepackage[utf8]{inputenc}
\usepackage[T1]{fontenc}
\usepackage{amsmath,amssymb,amsfonts,amsthm,mathtools}
\usepackage{booktabs}
\usepackage{enumitem}
\usepackage{xcolor}
\usepackage{hyperref}
\usepackage{multirow}
\usepackage{array}
\usepackage{graphicx}
\usepackage{subcaption}
\usepackage{microtype}
\usepackage{url}
\usepackage{tikz}
\usepackage{pgfplots}
\usepackage{tabularx}
\usepackage{makecell}
\usepackage{wrapfig}
\usepackage{hyperref}

\definecolor{cycleblue}{RGB}{48,99,178}
\definecolor{cycleorange}{RGB}{218,124,48}
\definecolor{cycleteal}{RGB}{0,132,120}
\definecolor{cyclered}{RGB}{185,73,68}
\definecolor{cyclegray}{RGB}{118,118,118}
\definecolor{cyclepurple}{RGB}{108,85,165}

\newcommand{\softbadge}[2]{%
	\tikz[baseline=(X.base)]{%
		\node[
		rounded corners=1.1pt,
		inner xsep=3.0pt,
		inner ysep=1.25pt,
		fill=#1!10,
		draw=#1!28,
		line width=0.25pt
		] (X) {\textcolor{#1!78!black}{\scriptsize\bfseries #2}};
	}%
}

\newcommand{\thinmetricbar}[4]{%
	\begin{tikzpicture}[baseline=-0.55ex]
		\pgfmathsetmacro{\barw}{1.46}
		\pgfmathsetmacro{\val}{max(0,min(\barw,\barw*(#1-#2)/(#3-#2)))}
		\pgfmathsetmacro{\chance}{max(0,min(\barw,\barw*(0.5-#2)/(#3-#2)))}
		\fill[gray!12,rounded corners=0.35pt] (0,0) rectangle (\barw,0.105);
		\fill[#4!72,rounded corners=0.35pt] (0,0) rectangle (\val,0.105);
		\draw[gray!38,line width=0.22pt] (0,0) rectangle (\barw,0.105);
		\draw[black!45,line width=0.30pt] (\chance,-0.022) -- (\chance,0.127);
	\end{tikzpicture}%
}

\newcommand{\auroccell}[3]{%
	\makecell[c]{\textbf{#1}\\[-0.55mm]\thinmetricbar{#2}{0.5}{1.0}{#3}}%
}

\newcommand{\emptycell}[1]{%
	\makecell[c]{\textcolor{cyclegray!85!black}{#1}\\[-0.4mm]\softbadge{cyclegray}{n/a}}%
}

\pgfplotsset{
	compat=1.18,
	cyclebar/.style={
		ybar,
		axis line style={draw=gray!55},
		tick style={draw=gray!55},
		grid=major,
		major grid style={draw=gray!16},
		label style={font=\small},
		tick label style={font=\small},
		legend style={draw=none, fill=none, font=\small},
		scaled ticks=false,
		axis x line*=bottom,
		axis y line*=left,
		enlarge x limits=0.12,
	},
	cycleline/.style={
		axis line style={draw=gray!55},
		tick style={draw=gray!55},
		grid=major,
		major grid style={draw=gray!16},
		label style={font=\small},
		tick label style={font=\small},
		legend style={draw=none, fill=none, font=\scriptsize},
		scaled ticks=false,
		axis x line*=bottom,
		axis y line*=left,
	},
}

\hypersetup{
	colorlinks=true,
	linkcolor=blue!55!black,
	citecolor=blue!55!black,
	urlcolor=blue!55!black
}

\newtheorem{theorem}{Theorem}

\newtheorem{proposition}{Proposition}

\newtheorem{corollary}{Corollary}

\newcommand{\E}{\mathbb E}
\newcommand{\Prb}{\mathbb P}
\newcommand{\R}{\mathbb R}

\newcommand{\calU}{\mathcal U}
\newcommand{\calV}{\mathcal V}

\newcommand{\Dcyc}[1][]{\Delta_{\mathrm{cyc}#1}}
\newcommand{\Acccyc}{\mathrm{CycleAcc}}

\setlist{nosep,leftmargin=1.25em}

\title{Can the Parts Fool the Test?\\
	Counterfactual Pairing Cycles for Relational OOD}

\author{%
	Yibo Zhou\\
	Beihang University\\
	\texttt{ybzhou@buaa.edu.cn}
	\And
	Yirong Xiang\thanks{Corresponding author.}\\
	Air Force Medical Center\\
	\texttt{yirong\_xiang@163.com}
}

\begin{document}
	
	\maketitle
	
	\begin{abstract}
		Relational out-of-distribution errors are easy to describe and hard to evaluate cleanly. An image may be natural, a caption may be fluent, and all mentioned objects may be common, yet the caption may describe the wrong relation in the image. Many evaluations intended to test this behavior can be solved without checking the pairing at all, because negatives differ in caption style, template, length, object frequency, or image statistics. We propose counterfactual pairing cycles, a blocked evaluation primitive that reuses the same two images and the same two captions on both sides of the comparison, changing only which image is paired with which caption. This makes image-only, text-only, and additive component shortcuts cancel exactly in finite samples. The same contrast also gives a training loss for relational outlier exposure and a diagnostic for bilinear vision-language scores. Across synthetic controls, GQA, and COCO, cycles separate coarse image-caption matching from directional role-binding failures, while also identifying geometric size relations that remain strongly visible in global features. In a shortcut-contaminated COCO outlier-exposure stress test, a model selected by a near-perfect shortcut validation AUROC retains substantial relational signal but remains below cycle training on exact cycles; cycle training improves CycleAcc from \(0.882\) to \(0.927\). The main lesson is simple: to test whether a model understands a pairing, the benchmark should keep the parts fixed, change the pairing, and ask only about coupling. The official implementation is available \href{https://github.com/SDret/Official-Implementation-of-Neurips-2026-paper-Can-the-Parts-Fool-the-Test-/tree/main}{online}.
	\end{abstract}
	
	\section{Introduction}
	
	A wrong image-caption pair can look perfectly ordinary in isolation. The image is a real photograph. The caption is grammatical. The objects are familiar. The failure appears only when the two are put together: the caption assigns an attribute to the wrong object, reverses a spatial relation, or describes an event that belongs to a different image.
	
	This is the relational OOD problem we study. It is not an anomaly of the image alone or the caption alone. It is an anomaly of their coupling. This distinction has become important as vision-language models are deployed as general-purpose perceptual systems. Recent evaluations show that strong multimodal models can struggle with fine-grained composition, natural adversarial examples, spatial reasoning, and object-relation binding~\cite{chen2024mmstar,jiang2024naturalbench,hua2024mmcomposition,kamath2024hard,chen2024spatialvlm,jia2025omnispatial,stogiannidis2025mind}. These failures are small in form but large in consequence: a model may recognize the objects and still miss who is doing what to whom, which object is on the left, or which caption belongs to the image.
	
	The difficulty is not only that relational errors are hard for models. They are also hard to evaluate cleanly. Many benchmarks that test multimodal reasoning contain easier signals. If negative captions are shorter, come from a different generator, use a special template, or mention rarer objects, a text-only detector can succeed. If negative images differ in style or category distribution, an image-only detector can succeed. Similar concerns have motivated efforts to build vision-indispensable, adversarial, bidirectional, and composition-focused benchmarks for vision-language models~\cite{chen2024mmstar,jiang2024naturalbench,miranda2024bivlc,hua2024mmcomposition,kamath2024hard,schneider2024m5}. Related work on attribute recognition and spurious correlation likewise shows that label imbalance, component bias, and object co-occurrence can provide alternatives to learning the intended coupling~\cite{ijcai2023p203,pmlr-v235-zhou24j,zhou2025cooccurrence,zhou2025heterogeneous,zhou2026spurious}. The common lesson is that high benchmark performance can mislead when evaluation does not control what information the model is allowed to use.
	
	This issue is especially sharp for OOD detection. Classical and reconstruction-based OOD methods already show that the statistic used for training and validation determines what kind of abnormality is learned~\cite{hendrycks2017baseline,liang2018enhancing,liu2020energy,zhou2022rethinking}. Modern OOD methods increasingly use vision-language representations, negative labels, prompt tuning, and multimodal cues~\cite{jiang2024neglabel,zhang2024dualpattern,yu2024selfcalibrated,dong2024multiood,schneider2024m5}. These methods are valuable, but the standard OOD framing often asks whether an example looks unlike the training distribution. Relational OOD asks a different question. The image and the caption may both look in-distribution. What is out-of-distribution is the pairing. A detector that succeeds by scoring the image and caption separately may therefore solve the wrong problem.
	
	We propose a small change in the unit of evaluation. Instead of comparing arbitrary positive and negative pairs, we compare two ways of pairing the same components. Take two valid image-caption pairs. Now keep the two images and the two captions fixed, but swap the captions across images. This creates a complete four-entry block: two original pairs and two swapped pairs. We retain the block only when the swapped pairs are verified invalid by the construction annotations or filters.
	
	The test then asks a single question: does the model prefer the original pairing over the swapped pairing? The positive side and the negative side contain exactly the same images and exactly the same captions. Only the assignment between images and captions changes. As a result, image-only cues, text-only cues, and any additive combination of such component cues cancel inside the block. The benchmark no longer asks whether either side looks unusual. It asks whether the two sides match. This blocked comparison is the central object of the paper. We call it a \emph{counterfactual pairing cycle}. ``Counterfactual'' refers to the swapped pairings: the caption is fluent and the image is natural, but their association is wrong. ``Cycle'' refers to the complete four-entry structure: each image is paired once with its true caption and once with the other caption, and each caption is paired once with its true image and once with the other image. This symmetry is what makes the test insensitive to marginal shortcuts. We make three contributions:
	
	\begin{enumerate}[leftmargin=1.35em,itemsep=2pt,topsep=3pt]
		\item We formalize \emph{counterfactual pairing cycles} as a blocked evaluation primitive for relational OOD. In every block, the positive and negative sides have identical image and caption marginals.
		\item We characterize the finite-sample balance condition required to remove all additive image-only and text-only scores: a signed comparison cancels such scores if and only if every image and every caption has zero signed degree. The \(2\times2\) counterfactual pairing cycle is the minimal nondegenerate block satisfying this condition.
		\item We use the same contrast as a training loss for relational outlier exposure and as a diagnostic for bilinear dual-encoder scores. The bilinear cycle score factorizes into paired image and text differences. Also, we evaluate the protocol on synthetic controls, GQA, and COCO. Cycles separate coarse matching from several strict role-binding failures and show why shortcut validation can overstate relational transfer.
	\end{enumerate}
	
	\section{The cycle test}
	
	\subsection{Cycles and cycle accuracy}
	
	Let \(U\in\calU\) and \(V\in\calV\) be two components, such as an image and a caption. A counterfactual pairing cycle is a complete block
	\[
	\begin{array}{c|cc}
		& v_1 & v_2 \\
		\hline
		u_1 & \text{valid} & \text{invalid swap} \\
		u_2 & \text{invalid swap} & \text{valid}
	\end{array}
	\]
	constructed from two valid pairs \((u_1,v_1)\) and \((u_2,v_2)\). The swapped pairs are retained only if the dataset annotations or construction filters mark them invalid for the intended relation.
	
	Given a scalar compatibility score \(s:\calU\times\calV\to\R\), define the cycle contrast on block \(C\) by
	\[
	\Delta_C(s)
	=
	s(u_1,v_1)+s(u_2,v_2)-s(u_1,v_2)-s(u_2,v_1).
	\]
	The primary metric for exact-cycle experiments is
	\[
	\Acccyc(s)
	=
	\Prb_C\{\Dcyc[,C](s)>0\}
	+
	\frac12\Prb_C\{\Dcyc[,C](s)=0\}.
	\]
	Here and throughout the empirical sections, \(C\) denotes a retained exact cycle drawn uniformly from the evaluation split; when a population cycle distribution is used instead, it is stated explicitly. This half-tie convention makes every additive score exactly chance on every
	exact-cycle split in exact arithmetic. Reported results use the fixed
	absolute numerical tie tolerance specified in
	Appendix~\ref{app:cycle-evaluation}.
	
	\subsection{Why the shortcut disappears}
	
	The cancellation used by cycles is not specific to a particular model class. It is a finite-sample balance condition on the comparison itself. Any linear comparison of evaluated pairs can be viewed as a signed bipartite graph whose left vertices are images, whose right vertices are captions, and whose signed edges are pair evaluations. Let \(E\) be a finite multiset of evaluated pairs \(e=(u_e,v_e)\), with weights \(w_e\in\mathbb R\). Positive weights correspond to the side of the comparison that should receive higher score, and negative weights to the opposite side. Define
	\[
	T_w(s)=\sum_{e\in E} w_e s(u_e,v_e).
	\]
	For each image \(u\) and caption \(v\), define their signed degrees
	\[
	d_w(u)=\sum_{e:u_e=u} w_e,
	\qquad
	d_w(v)=\sum_{e:v_e=v} w_e .
	\]
	
	\begin{theorem}[Finite-sample additive invariance]
		\label{thm:signed_degree}
		The contrast \(T_w\) cancels all image-only and text-only additive scores, meaning
		\[
		T_w(a(u)+b(v))=0
		\]
		for all functions \(a\) and \(b\), if and only if
		\[
		d_w(u)=0\quad\text{for every image }u,
		\qquad
		d_w(v)=0\quad\text{for every caption }v .
		\]
	\end{theorem}
	
	\begin{proof}
		For an additive score \(s(u,v)=a(u)+b(v)\),
		\[
		\begin{aligned}
			T_w(s)
			&=
			\sum_{e\in E}w_e a(u_e)+\sum_{e\in E}w_e b(v_e)\\
			&=
			\sum_u d_w(u)a(u)+\sum_v d_w(v)b(v).
		\end{aligned}
		\]
		If all signed degrees are zero, this expression is zero for all \(a,b\). Conversely, if some image \(u_0\) has \(d_w(u_0)\neq0\), choosing \(a(u_0)=1\), \(a(u)=0\) for \(u\neq u_0\), and \(b\equiv0\) makes \(T_w(a+b)=d_w(u_0)\neq0\). The same argument applies to captions.
	\end{proof}

	This theorem gives the design principle behind exact cycles. A benchmark comparison removes arbitrary component-only additive scores exactly when every component appears with the same total weight on both sides of the comparison. Thus exact marginal cancellation is a property of the signed evaluation graph for the reported comparison, not merely of having paired examples or approximately matched dataset marginals. Counterfactual cycles enforce this balance locally, within each retained block. After merging duplicate evaluations and excluding the identically zero contrast, the counterfactual pairing cycle is the smallest nondegenerate instance of this principle involving two distinct images and two distinct captions. It uses the four signed edges
	\[
	+(u_1,v_1),\quad +(u_2,v_2),\quad -(u_1,v_2),\quad -(u_2,v_1),
	\]
	so every image and every caption has signed degree zero. The resulting contrast is
	\[
	\Dcyc[,C](s)
	=
	s(u_1,v_1)+s(u_2,v_2)-s(u_1,v_2)-s(u_2,v_1).
	\]
	
	\begin{corollary}[Atomic \(2\times2\) invariant]
		\label{cor:atomic_cycle}
		On a complete \(2\times2\) block, every nonzero linear contrast that cancels all additive image and text scores is proportional to
		\[
		(1,-1,-1,1).
		\]
		Equivalently, up to scale, the cycle contrast is the unique nonzero linear statistic on the block that removes all additive component main effects.
	\end{corollary}
	
	\begin{proof}
		Apply Theorem~\ref{thm:signed_degree} to the four coefficients \((c_{11},c_{12},c_{21},c_{22})\). Zero signed degree at the two image vertices gives
		\[
		c_{11}+c_{12}=0,\qquad c_{21}+c_{22}=0,
		\]
		and zero signed degree at the two caption vertices gives
		\[
		c_{11}+c_{21}=0,\qquad c_{12}+c_{22}=0.
		\]
		Thus
		\[
		(c_{11},c_{12},c_{21},c_{22})
		=
		\lambda(1,-1,-1,1)
		\]
		for some \(\lambda\).
	\end{proof}
	
	\begin{corollary}[Additive scores are chance under exact cycles]
		\label{cor:additive_zero}
		For every score of the form \(s(u,v)=a(u)+b(v)\),
		\[
		\Dcyc[,C](s)=0
		\]
		on every exact cycle \(C\). Therefore \(\Acccyc(s)=1/2\) under the half-tie convention.
	\end{corollary}
	
	The guarantee is finite-sample and applies to additive component-level nuisance directions: such terms cancel exactly within each retained block. They do not rule out non-additive pair-dependent artifacts, such as object co-occurrence, template-specific interactions, or annotation-rule effects. Therefore high \(\Acccyc\) should be interpreted as evidence of sensitivity to pair-dependent information, not as a certificate of full semantic understanding of the intended relation. The same algebra also gives a useful population interpretation. A cycle is not only a finite-sample trick for canceling shortcuts; it estimates how much a score prefers the observed joint pairing over an independent recombination with the same marginals.
	
	\begin{proposition}[Cycles estimate the interaction advantage]
		\label{prop:cycle_population}
		Let \(P\) be the distribution of valid pairs, with marginals \(\mu\) and \(\nu\), and let \(R=\mu\otimes\nu\). Assume \(P\ll R\), \(dP/dR\in L^2(R)\), and \(s\in L^2(R)\). Draw two valid pairs
		\[
		(U_1,V_1),(U_2,V_2)\stackrel{\mathrm{iid}}{\sim}P
		\]
		and form the swapped cycle \(C\) without any additional filtering. Define the ANOVA interaction residual under \(R\) by
		\[
		s_{\rm int}(u,v)
		=
		s(u,v)
		-
		\E_R[s(U,V)\mid U=u]
		-
		\E_R[s(U,V)\mid V=v]
		+
		\E_R s(U,V).
		\]
		Then
		\[
		\E[\Delta_C(s)]
		=
		2\big(\E_P s(U,V)-\E_R s(U,V)\big)
		=
		2\E_P s_{\rm int}(U,V).
		\]
		Moreover,
		\[
		\Delta_C(s)=\Delta_C(s_{\rm int})
		\]
		for every realized cycle.
	\end{proposition}
	
	\begin{proof}
		Since the two valid pairs are independent and have marginals \(\mu\) and \(\nu\), the swapped pair \((U_1,V_2)\) has distribution \(\mu\otimes\nu=R\), and so does \((U_2,V_1)\). Hence
		\[
		\E[\Delta_C(s)]
		=
		2\E_P s(U,V)-2\E_R s(U,V).
		\]
		The difference between \(s\) and \(s_{\rm int}\) is additive in \(U\) and \(V\), up to a constant. Additive terms cancel blockwise in the cycle contrast, which gives
		\[
		\Delta_C(s)=\Delta_C(s_{\rm int})
		\]
		for every realized cycle. Since \(P\) and \(R\) have the same marginals, additive terms have the same expectation under \(P\) and \(R\). Therefore
		\[
		\E_P s-\E_R s
		=
		\E_P s_{\rm int}.
		\]
		Combining the two identities gives the result.
	\end{proof}
	
	This population identity applies to unfiltered iid recombinations of valid pairs. In the exact-cycle benchmarks below, we additionally condition on the swapped pairs being invalid. That filtering preserves the blockwise additive-cancellation guarantee, because each retained block still contains the same two images and the same two captions on both sides, but it changes the population distribution of retained cycles. Therefore Proposition~\ref{prop:cycle_population} should be read as an unfiltered population interpretation rather than as an exact description of the filtered benchmark distribution. This motivates our experimental separation between coarse object-caption cycles and stricter relation cycles where both mentioned objects are present and the swapped relation is false.
	
	\section{Learning and diagnostics}
	
	\paragraph{Cycle loss.}
	The same contrast can be used for training via
	\[
	\mathcal L_{\rm cyc}(\theta;C)=\mathrm{softplus}(-\Dcyc[,C](s_\theta)).
	\]
	Because \(\Dcyc[,C](s+a(u)+b(v))=\Dcyc[,C](s)\) for every cycle, arbitrary image-only and text-only heads cannot improve this objective. This differs from a four-entry BCE over the two valid and two swapped pairs, which depends on absolute scores and can reward component shifts that carry no relational information.
	
	\paragraph{Bilinear scores.}
	Let \(\phi:\calU\to\R^{d_\phi}\) and \(\psi:\calV\to\R^{d_\psi}\) denote image and text embeddings, respectively. For a bilinear score
	\[
	s(u,v)=\phi(u)^\top M\psi(v),
	\qquad
	M\in\R^{d_\phi\times d_\psi},
	\]
	the cycle contrast factorizes as
	\[
	\Dcyc[,C](s)
	=
	[\phi(u_1)-\phi(u_2)]^\top M[\psi(v_1)-\psi(v_2)].
	\]
	Thus rank sweeps distinguish two bottlenecks: if the relevant image and text differences are present but spread across directions, increasing the rank of \(M\) can help; if the representations do not contain the needed difference, increasing rank will not recover it.
	
	\section{Constructing exact cycles}
	
	An exact cycle must satisfy two requirements: both swaps are invalid for the intended relation, and the full four-entry block is available. We therefore use conservative dataset-specific construction rules and then apply a complete-block filter. Each retained block contains two valid pairs and their two swaps, with labels \(1,1,0,0\); if any entry is missing, required annotations are unavailable, or a swap cannot be verified as invalid, the entire block is discarded. The construction principle is shared across datasets. For GQA~\cite{hudson2019gqa}, we instantiate caption-like facts from scene graphs and retain a block only when both swapped image-fact pairings are invalid under the annotation-derived predicate. For COCO object cycles, a swapped caption is invalid if it mentions a selected MSCOCO category~\cite{lin2014microsoft}, under a fixed synonym table, that is has no corresponding annotation in the target image. For COCO relation cycles, captions are generated from instance boxes; the source relation must hold under a fixed geometric margin, and in strict splits the target image must contain both mentioned categories while making the swapped relation false under the same rule. Construction annotations are used only for filtering and are not given to the detector. Because every retained block reuses the same two images and captions on both sides, the positive and negative image and caption marginals match exactly within each block and after aggregation. Caption-length marginals also match whenever caption strings are reused. Appendix~\ref{app:datasets} reports the corresponding total-variation audits.
	
	\begin{table}[t]
		\centering
		\caption{Synthetic controls. AUROC is reported for these unblocked two-sample diagnostics. The vertical tick in each mini-bar marks chance.}
		\label{tab:synthetic}
		\small
		\setlength{\tabcolsep}{3.8pt}
		\renewcommand{\arraystretch}{1.16}
		\begin{tabularx}{\linewidth}{
				@{}
				>{\raggedright\arraybackslash}p{0.18\linewidth}
				>{\centering\arraybackslash}p{0.205\linewidth}
				>{\centering\arraybackslash}p{0.205\linewidth}
				>{\raggedright\arraybackslash}X
				@{}
			}
			\toprule
			Control & Marginal score & Interaction score & Interpretation \\
			\midrule
			Pure interaction
			& \auroccell{\(0.5005\pm0.0072\)}{0.5005}{cyclegray}
			& \emptycell{--}
			& The marginal probe is at chance in this synthetic construction. \\
			Marginal shortcut
			& \auroccell{\(1.0000\pm0.0000\)}{1.0000}{cyclered}
			& \softbadge{cyclered}{shortcut}
			& The same probe becomes perfect once a marginal artifact is injected. \\
			Rank bottleneck
			& \auroccell{\(0.5023\pm0.0023\)}{0.5023}{cyclegray}
			& \auroccell{\(0.7658\pm0.0031\)}{0.7658}{cycleblue}
			& Increasing rank initially recovers interaction modes and then saturates. \\
			Binding representation
			& \auroccell{\(0.5000\)}{0.5000}{cyclegray}
			& \auroccell{\(0.9919\pm0.0006\)}{0.9919}{cycleblue}
			& Bag features discard binding, while pair-aware features recover it. \\
			\bottomrule
		\end{tabularx}
	\end{table}

	\section{Experiments}
	
	Our experiments are diagnostic. We first check that the cycle construction preserves the signed-degree balance that removes additive component-level leakage in settings where the answer is known. We then use exact cycles to examine what frozen CLIP-like representations expose: coarse image-caption compatibility, geometric relations, or strict role binding. Finally, we test whether an outlier-exposure model selected by a shortcut validation signal transfers to relational cycles.
	
	Unless otherwise stated, raw-image experiments use frozen CLIP ViT-B/32 image and text features~\cite{radford2021learning}. We evaluate image-only, text-only, and additive logistic probes as implementation audits rather than competitive baselines: by Theorem~\ref{thm:signed_degree}, exact cycles force any additive score to tie, so chance performance verifies that the retained blocks, filtering, and feature availability have preserved the required signed-degree balance. We also evaluate zero-shot CLIP cosine; diagonal and bilinear interaction probes; and low-rank bilinear scores
	\[
	s_r(u,v)=(A^\top\phi(u))^\top(B^\top\psi(v)),
	\qquad
	A\in\R^{d_\phi\times r},\quad B\in\R^{d_\psi\times r},
	\]
	where \(A\) and \(B\) are learned projection matrices for the frozen image and text embeddings. All exact-cycle results report \(\Acccyc\). AUROC is used only for unblocked synthetic diagnostics and for the shortcut validation task.
	
	\subsection{Sanity checks: the audit behaves as intended}
	
	Synthetic controls isolate the algebra from dataset details. Table~\ref{tab:synthetic} shows four checks. The pure-interaction row is used only to verify that marginal scores have chance performance when the shift is purely in the coupling; the interaction score is therefore not reported for that row. Additive probes are chance on pure interactions, but become perfect when a marginal shortcut is injected. Rank-constrained interaction probes improve from \(0.6725\) at rank \(1\) to \(0.7658\) at rank \(8\), then slightly decrease to \(0.7575\) at rank \(16\). Bag representations discard binding information that pair-aware scores recover. These controls are unblocked two-sample diagnostics, so we report AUROC. The same signed-degree balance check holds on real exact-cycle splits: image-only, text-only, and additive probes are at chance under \(\Acccyc\), as required for retained complete blocks. When we deliberately contaminate text marginals in an unblocked diagnostic, text-only and additive probes become perfect. Thus the cycle metric removes additive component plausibility as a route to success and rewards pair-dependent signal instead. 
	
	\subsection{Exact-cycle diagnostics on GQA and COCO}
	
	Table~\ref{tab:diagnostics} summarizes the main exact-cycle diagnostics. The additive column is an implementation audit: by Theorem~\ref{thm:signed_degree}, it should be chance whenever the retained blocks preserve exact signed-degree balance. It checks the construction and filtering pipeline rather than serving as a competitive baseline. For COCO rows, the interaction ranges are zero-shot CLIP cosine ranges. For GQA rows, the ranges summarize the representative trained non-additive probes: diagonal, full bilinear, and rank-\(64\) low-rank bilinear scores. Additive image-only and text-only heads are excluded from the interaction column. The GQA binding-swap family did not yield enough validation cycles under the minimum split-size filter and is therefore not reported.
	
	The results separate several regimes. CLIP-like global features are strong on coarse image-caption matching and many object-level shifts. Directional left/right binding remains close to chance in the strict COCO splits. In contrast, same-pair relation changes, above/below relations, and especially size relations retain recoverable pair-dependent signal, with the strongest signal observed for size.
	
	\subsection{Strict spatial relations}
	
	Strict COCO relation cycles remove the easiest object-presence cue. The swapped caption is not invalid because an object is missing: both object categories appear in the target image, but the stated spatial or size relation is false under the box-geometry rule. This setting is a more direct test of role binding. Figure~\ref{fig:strict_combined} summarizes the strict-relation results. The table panel, Figure~\ref{fig:strict_table}, reports the per-family \(\Acccyc\) values, and the bar-plot panel, Figure~\ref{fig:strict_bars}, highlights the main comparison. Left/right swaps remain near chance for both CLIP cosine and the rank-\(64\) bilinear probe. Above/below is weak for zero-shot CLIP but reaches \(0.659\) with the trained rank-\(64\) probe. Same-pair relation changes remain near chance for CLIP and reach \(0.599\) with rank \(64\). Size retains the strongest signal: \(0.763\) for zero-shot CLIP and \(0.787\) for the rank-\(64\) probe. Thus strict directional role binding remains difficult, while geometric size information is substantially more visible in the frozen global representation.
	
\begin{table}[t]
	\centering
	\caption{Exact-cycle diagnostics on GQA and COCO. Values are exact
		CycleAcc. Additive probes are implementation consistency checks. For COCO,
		the interaction ranges are zero-shot CLIP cosine ranges only; trained
		bilinear and low-rank bilinear probes are reported separately. For GQA,
		the interaction ranges summarize the evaluated diagonal, full-bilinear,
		and rank-64 probes. The GQA binding family is omitted because it yields
		only 18--19 validation cycles per seed, below the prespecified minimum of
		30.}
	\label{tab:diagnostics}
	\small
	\setlength{\tabcolsep}{4.8pt}
	\renewcommand{\arraystretch}{1.14}
	\begin{tabularx}{\linewidth}{
			@{}
			>{\raggedright\arraybackslash}p{0.28\linewidth}
			c
			c
			>{\raggedright\arraybackslash}X
			@{}
		}
		\toprule
		Split group & Additive audit & Interaction scores & Reading \\
		\midrule
		
		\multicolumn{4}{@{}l}{\textcolor{black!70}{\scriptsize\bfseries GQA scene-graph cycles}}\\[-0.1mm]
		Attribute / caption mismatch
		& \(0.500\pm0.000\)
		& \(0.594\)--\(0.954\)
		& Attribute coupling varies across families and can be strong. \\
		
		Relation / object swaps
		& \(0.500\pm0.000\)
		& \(0.500\)--\(0.835\)
		& Performance ranges from chance to substantial relation signal. \\
		
		\addlinespace[2pt]
		\midrule
		\addlinespace[1pt]
		
		\multicolumn{4}{@{}l}{\textcolor{black!70}{\scriptsize\bfseries COCO raw-image cycles}}\\[-0.1mm]
		Coarse coupling
		& \(0.500\pm0.000\)
		& \(0.997\)--\(1.000\)
		& Coarse image-caption matching is easy. \\
		
		Matched / neighbor hard
		& \(0.500\pm0.000\)
		& \(0.941\)--\(0.987\)
		& Coarse object signal remains strong. \\
		
		Box relation
		& \(0.500\pm0.000\)
		& \(0.903\)--\(0.945\)
		& Box-level geometry is visible. \\
		
		Strict relations
		& \(0.500\pm0.000\)
		& \(0.495\)--\(0.763\)
		& Directional relations are weak, while apparent size is visible. \\
		
		\bottomrule
	\end{tabularx}
\end{table}
	
	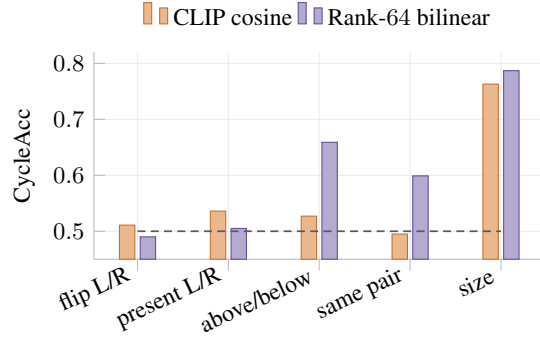
\begin{figure}[t]
		\centering
		\newlength{\strictpanelheight}
		\setlength{\strictpanelheight}{5.02cm}
		
		\begin{subfigure}[t]{0.43\linewidth}
			\vspace{0pt}
			\centering
			\begin{minipage}[t][\strictpanelheight][c]{\linewidth}
				\centering
				\small
				\setlength{\tabcolsep}{2.5pt}
				\renewcommand{\arraystretch}{1.28}
				\begin{tabular}{@{}lcc@{}}
					\toprule
					Family & CLIP & Rank-64 \\
					\midrule
					Flip L/R
					& \(0.511{\pm}0.012\)
					& \(0.490{\pm}0.020\) \\
					
					Present L/R
					& \(0.536{\pm}0.010\)
					& \(0.505{\pm}0.036\) \\
					
					Above/below
					& \(0.527{\pm}0.005\)
					& \(0.659{\pm}0.024\) \\
					
					Same-pair rel.
					& \(0.495{\pm}0.002\)
					& \(0.599{\pm}0.016\) \\
					
					Present size
					& \(0.763{\pm}0.010\)
					& \(\mathbf{0.787{\pm}0.010}\) \\
					\bottomrule
				\end{tabular}
			\end{minipage}
			\caption{Per-family \(\Acccyc\). Additive audits are \(0.500\pm0.000\) for all reported rows.}
			\label{fig:strict_table}
		\end{subfigure}
		\hfill
		\begin{subfigure}[t]{0.54\linewidth}
			\vspace{0pt}
			\centering
			\begin{minipage}[t][\strictpanelheight][c]{\linewidth}
				\centering
				\begin{tikzpicture}
					\begin{axis}[
						cyclebar,
						width=\linewidth,
						height=4.32cm,
						bar width=5.8pt,
						ymin=0.45,
						ymax=0.82,
						ylabel={CycleAcc},
						symbolic x coords={flip L/R,present L/R,above/below,same pair,size},
						xtick=data,
						x tick label style={rotate=25,anchor=east,font=\small},
						y tick label style={font=\small},
						label style={font=\small},
						legend style={
							at={(0.5,1.06)},
							anchor=south,
							legend columns=2,
							draw=none,
							font=\small
						},
						enlarge x limits=0.12,
						]
						\addplot[fill=cycleorange!55,draw=cycleorange!90!black] coordinates {
							(flip L/R,0.511)
							(present L/R,0.536)
							(above/below,0.527)
							(same pair,0.495)
							(size,0.763)
						};
						\addplot[fill=cyclepurple!50,draw=cyclepurple!90!black] coordinates {
							(flip L/R,0.490)
							(present L/R,0.505)
							(above/below,0.659)
							(same pair,0.599)
							(size,0.787)
						};
						\draw[black!65,densely dashed,line width=0.65pt]
						(axis cs:flip L/R,0.5) -- (axis cs:size,0.5);
						\legend{CLIP cosine,Rank-\(64\) bilinear}
					\end{axis}
				\end{tikzpicture}
			\end{minipage}
			\caption{Visualization of the same strict relation results. The dashed line marks chance.}
			\label{fig:strict_bars}
		\end{subfigure}
		
		\caption{Strict COCO relation cycles separate directional role binding from recoverable geometric cues. Left/right remains near chance, while size is strongly exposed and above/below and same-pair relation changes retain moderate trained-probe signal.}
		\label{fig:strict_combined}
	\end{figure}
	
	\begin{table}[t]
		\centering
		\caption{Low-rank bilinear probes. Values are \(\Acccyc\). Rank helps for recoverable interactions but does not rescue strict present left/right.}
		\label{tab:rank}
		\small
		\setlength{\tabcolsep}{6pt}
		\renewcommand{\arraystretch}{1.08}
		\begin{tabular}{@{}lcccc@{}}
			\toprule
			Split & Rank 1 & Rank 4 & Rank 16 & Rank 64 \\
			\midrule
			Object-matched COCO
			& \(0.784\pm0.018\)
			& \(0.905\pm0.011\)
			& \(0.936\pm0.009\)
			& \(0.959\pm0.001\) \\
			
			Text-neighbor hard
			& \(0.720\pm0.020\)
			& \(0.797\pm0.026\)
			& \(0.858\pm0.018\)
			& \(0.888\pm0.006\) \\
			
			Box left/right
			& \(0.707\pm0.029\)
			& \(0.812\pm0.005\)
			& \(0.843\pm0.005\)
			& \(0.858\pm0.019\) \\
			
			\midrule
			Strict present left/right
			& \(0.514\pm0.014\)
			& \(0.518\pm0.017\)
			& \(0.489\pm0.025\)
			& \(0.505\pm0.036\) \\
			
			Strict same-pair relation diff.
			& \(0.567\pm0.019\)
			& \(0.614\pm0.020\)
			& \(0.587\pm0.023\)
			& \(0.599\pm0.016\) \\
			
			Strict size
			& \(0.687\pm0.023\)
			& \(0.745\pm0.014\)
			& \(0.796\pm0.014\)
			& \(0.787\pm0.010\) \\
			\bottomrule
		\end{tabular}
	\end{table}
	
\subsection{Rank sweeps: bilinear recoverability}
	
	The bilinear factorization suggests a simple diagnostic. If a relation is present in the representation but spread across several directions, increasing the rank of the bilinear score should help. If the representation does not expose the needed relation, increasing rank should not help. Table~\ref{tab:rank} gives the exact sweep, and Figure~\ref{fig:rank_sweep} separates the regimes. Object-matched, text-neighbor, and box-relation splits improve substantially with rank. Strict present left/right remains near chance even at rank \(64\), indicating limited accessibility to the evaluated bilinear probes. Strict same-pair relation differences improve modestly, while strict size shows a large recoverable interaction signal.
	
	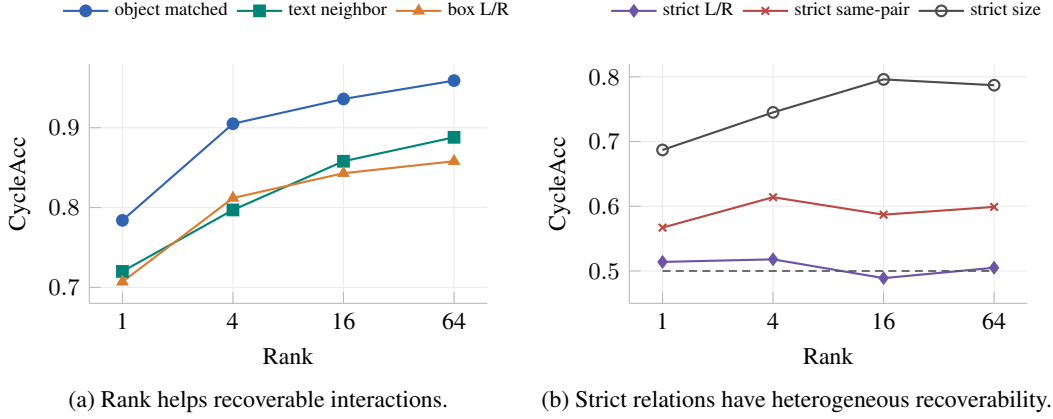
\begin{figure}[t]
		\centering
		\begin{subfigure}[t]{0.49\linewidth}
			\centering
			\begin{tikzpicture}
				\begin{axis}[
					cycleline,
					width=\linewidth,
					height=4.75cm,
					xlabel={Rank},
					ylabel={CycleAcc},
					xmode=log,
					log basis x={2},
					xtick={1,4,16,64},
					xticklabels={1,4,16,64},
					ymin=0.68,
					ymax=0.98,
					label style={font=\small},
					tick label style={font=\small},
					x tick label style={font=\small},
					y tick label style={font=\small},
					legend style={
						at={(0.5,1.14)},
						anchor=south,
						legend columns=3,
						draw=none,
						fill=none,
						font=\scriptsize,
						/tikz/every even column/.append style={column sep=1pt}
					},
					legend cell align={left},
					]
					\addplot[cycleblue,thick,mark=*] coordinates {(1,0.784) (4,0.905) (16,0.936) (64,0.959)};
					\addplot[cycleteal,thick,mark=square*] coordinates {(1,0.720) (4,0.797) (16,0.858) (64,0.888)};
					\addplot[cycleorange,thick,mark=triangle*] coordinates {(1,0.707) (4,0.812) (16,0.843) (64,0.858)};
					\legend{object matched,text neighbor,box L/R}
				\end{axis}
			\end{tikzpicture}
			\caption{Rank helps recoverable interactions.}
		\end{subfigure}
		\hfill
		\begin{subfigure}[t]{0.49\linewidth}
			\centering
			\begin{tikzpicture}
				\begin{axis}[
					cycleline,
					width=\linewidth,
					height=4.75cm,
					xlabel={Rank},
					ylabel={CycleAcc},
					xmode=log,
					log basis x={2},
					xtick={1,4,16,64},
					xticklabels={1,4,16,64},
					ymin=0.45,
					ymax=0.82,
					label style={font=\small},
					tick label style={font=\small},
					x tick label style={font=\small},
					y tick label style={font=\small},
					legend style={
						at={(0.5,1.14)},
						anchor=south,
						legend columns=3,
						draw=none,
						fill=none,
						font=\scriptsize,
						/tikz/every even column/.append style={column sep=1pt}
					},
					legend cell align={left},
					]
					\addplot[cyclepurple,thick,mark=diamond*] coordinates {(1,0.514) (4,0.518) (16,0.489) (64,0.505)};
					\addplot[cyclered,thick,mark=x] coordinates {(1,0.567) (4,0.614) (16,0.587) (64,0.599)};
					\addplot[black!70,thick,mark=o] coordinates {(1,0.687) (4,0.745) (16,0.796) (64,0.787)};
					\draw[black!65,densely dashed,line width=0.65pt]
					(axis cs:1,0.5) -- (axis cs:64,0.5);
					\legend{strict L/R,strict same-pair,strict size}
				\end{axis}
			\end{tikzpicture}
			\caption{Strict relations have heterogeneous recoverability.}
		\end{subfigure}
\caption{Rank sweeps characterize recoverability within the evaluated bilinear probe class. Some interactions improve with rank, while strict present left/right remains near chance.}
		\label{fig:rank_sweep}
	\end{figure}
	
	\subsection{Shortcut-contaminated outlier exposure}
	
	The last experiment tests the transfer risk predicted by the interaction view. Training and validation cycles are built from image-disjoint subsets of COCO \texttt{train2017}; test cycles are built from COCO \texttt{val2017}. Each exact cycle contains two images sharing an anchor object but differing in a second object. Diagonal captions describe true object pairs, and swapped captions are invalid because the variant object is absent. Both methods use the same frozen CLIP ViT-B/32 features and the same rank-\(64\) score class with additive heads:
	\[
	s_\theta(u,v)
	=
	(A^\top\phi(u))^\top(B^\top\psi(v))
	+
	\alpha^\top\phi(u)
	+
	\beta^\top\psi(v)
	+
	b,
	\]
	where \(A\in\R^{d_\phi\times64}\), \(B\in\R^{d_\psi\times64}\), \(\alpha\in\R^{d_\phi}\), \(\beta\in\R^{d_\psi}\), and \(b\in\R\) are learned parameters.
	
	The additive heads are included deliberately: ordinary OE can use component shortcuts if the training and validation signals make them predictive. Cycle OE trains on exact counterfactual cycles with
	\(\mathrm{softplus}(-\Dcyc[,C](s_\theta))\)
	and selects on held-out exact validation cycles. Shortcut-contaminated OE trains with BCE. Its positives are true image-caption pairs, and its negatives are produced by prefixing the true caption with the marker ``this is an out of distribution negative example.'' The prefix creates a clean text-marginal shortcut. OE validation uses the same marked matched-caption negatives, whereas exact-cycle test blocks contain only the original unmarked captions. This deliberately strong marker provides a controlled stress test of whether validation success on a marginal artifact transfers to relational cycles.
	
	Table~\ref{tab:shortcut_oe} shows the outcome. Shortcut OE nearly solves its own validation task, as does a text-only probe. Its exact-cycle accuracy remains substantial at \(0.8817\), but is consistently below Cycle OE at \(0.9270\). Thus shortcut validation overstates relational transfer, even though marker-contaminated BCE does not suppress all pair-dependent learning in the bilinear component. The exposed marker is sufficient for validation, but the marker-augmented captions still preserve the original caption semantics and the positive examples remain valid image-caption pairs. As a result, the learned interaction branch retains useful relational signal. A four-entry BCE baseline trained on exact cycles reaches \(0.8853\), close to the shortcut-trained model and below Cycle OE.
	
	\begin{table}[t]
		\centering
		\caption{Shortcut-contaminated COCO OE stress test. Shortcut validation uses AUROC; relational transfer uses exact-cycle \(\Acccyc\).}
		\label{tab:shortcut_oe}
		\small
		\setlength{\tabcolsep}{6pt}
		\renewcommand{\arraystretch}{1.08}
		\begin{tabular}{@{}lcc@{}}
			\toprule
			Method / audit & Shortcut val AUROC & Exact \(\Acccyc\) \\
			\midrule
			Cycle OE
			& --
			& \(\mathbf{0.9270\pm0.0075}\) \\
			
			Shortcut-contaminated OE
			& \(0.9998\pm0.0002\)
			& \(0.8817\pm0.0021\) \\
			
			\midrule
			Text-only shortcut probe
			& \(1.0000\pm0.0000\)
			& -- \\
			
			Image-only shortcut probe
			& \(0.5000\pm0.0000\)
			& -- \\
			\bottomrule
		\end{tabular}
	\end{table}
	
	\subsection{Transfer alignment}
	
	We next measure whether interaction structure learned on one exact-cycle
	family transfers to another. For each family $f$, let
	\[
	G_f
	=
	\frac{1}{|\mathcal C_f^{\mathrm{val}}|}
	\sum_{C\in\mathcal C_f^{\mathrm{val}}}
	\left[\phi(u_1)-\phi(u_2)\right]
	\left[\psi(v_1)-\psi(v_2)\right]^\top ,
	\]
	where $\mathcal C_f^{\mathrm{val}}$ contains the held-out validation cycles
	for family $f$. The alignment between two families is the normalized
	Frobenius inner product
	\[
	A_{ij}
	=
	\frac{\langle G_i,G_j\rangle_F}
	{\|G_i\|_F\|G_j\|_F}.
	\]
	
	For every ordered off-diagonal family pair $(i,j)$, we train a rank-64
	bilinear probe on family $i$ and evaluate its exact CycleAcc on the held-out
	test cycles of family $j$. Alignments are computed on validation cycles,
	while transfer performance is computed on test cycles. Correlations are
	computed within each seed, and we report the mean and sample standard
	deviation across seeds. CycleAcc compares the two complete assignments within a block. It does not
	require every directional image-to-caption or caption-to-image comparison
	to be individually correct, and it is not a certificate of full semantic
	understanding.
	
	COCO contains $182$ ordered off-diagonal family pairs per seed and gives
	\[
	\mathrm{Pearson}=0.931\pm0.006,
	\qquad
	\mathrm{Spearman}=0.876\pm0.016.
	\]
	GQA contains $42$ ordered off-diagonal pairs per seed among its available
	families and gives
	\[
	\mathrm{Pearson}=0.964\pm0.006,
	\qquad
	\mathrm{Spearman}=0.765\pm0.048.
	\]
	Thus, cycle-difference alignment is strongly associated with transfer in
	these diagnostics. Because ordered family pairs share families and trained
	probes, these correlations are descriptive rather than independent-sample
	inferential estimates.
	
	\section{Related work}
	
	Classical OOD detection uses confidence, density, reconstruction, or energy-based scores~\cite{hendrycks2017baseline,liang2018enhancing,liu2020energy,zhou2022rethinking}, while outlier exposure learns from auxiliary negatives~\cite{hendrycks2019deep}. These settings typically ask whether an example looks unlike the training distribution. Relational OOD instead asks whether two individually plausible components are paired correctly, making marginal shortcuts especially harmful: validation success from caption style, length, object frequency, or image-distribution artifacts may overestimate transfer to wrong-pairing errors. Similar shortcut issues arise in hypothesis-only and question-only baselines for NLI and VQA~\cite{gururangan2018annotation,poliak2018hypothesis,agrawal2018dont}. Work on pedestrian attribute recognition has likewise shown that label imbalance, co-occurrence bias, and feature resampling can determine whether a model learns the intended attribute rather than a correlated proxy~\cite{pmlr-v235-zhou24j,ijcai2023p203,zhou2025cooccurrence,zhou2025heterogeneous}; related analysis asks when spurious-correlation removal is learnable at all~\cite{zhou2026spurious}.
	
	Compositional vision-language benchmarks such as Winoground, ARO, and SugarCrepe use paired comparisons to test word order, binding, and relational understanding~\cite{thrush2022winoground,yuksekgonul2023when,hsieh2023sugarcrepe}. BiVLC extends this evaluation to bidirectional retrieval~\cite{miranda2024bivlc}, while CLoVe directly improves compositional language encoding in contrastive models~\cite{castro2024clove}. Our distinction is the invariant statistic used to score the comparison. Theorem~\ref{thm:signed_degree} shows that exact cancellation of image-only and text-only scores requires zero signed degree for every image and caption in the reported comparison. Counterfactual pairing cycles enforce this condition locally with a complete \(2\times2\) block and score it by the signed interaction contrast. Thus the primitive is not a pooled matched/mismatched ranking set, but a finite-sample marginally balanced test for relational OOD. The training and diagnostic uses connect to contrastive learning~\cite{oord2018representation,chen2020simple} and interaction or dependence testing~\cite{gretton2005measuring,gretton2012kernel,sejdinovic2013equivalence}, but target pair-level relational scoring under exact marginal cancellation.
	
\section{Conclusion}

Relational out-of-distribution evaluation requires evidence about the
pairing, not only evidence that each component is plausible.
Counterfactual pairing cycles provide a finite-sample block design in
which image-only, text-only, and arbitrary additive component scores
cancel exactly. The same interaction contrast yields a training objective
and a family of diagnostics for bilinear vision-language scores. Across synthetic controls, COCO, and GQA, the diagnostics separate coarse
object compatibility from stricter spatial and role-binding structure.
Coarse COCO cycles remain easy, strict left/right is weak for the
evaluated bilinear probes, and apparent size is substantially more
visible. In the marker-controlled COCO study, Cycle OE obtains
\(0.9270\pm0.0075\) exact CycleAcc, while shortcut-contaminated BCE
obtains \(0.8817\pm0.0021\). Because these protocols differ in objective,
validation criterion, and marker exposure, the comparison should not be
read as isolating a single causal factor. Exact CycleAcc measures sensitivity to pair-dependent structure within a
complete block. It does not certify full semantic understanding, and its
conclusions depend on the cycle predicates and representations being
evaluated. The public implementation enables the same audits to be applied
to additional encoders, benchmarks, and relation families.
	
	\begin{ack}
		The authors received no specific funding for this work and declare no competing interests.
	\end{ack}
	
\paragraph{Code availability.}
The official implementation is hyperlinked in the abstract.
	
	\bibliographystyle{plainnat}
	\bibliography{ref}

@inproceedings{hendrycks2017baseline,  
	title     = {A Baseline for Detecting Misclassified and Out-of-Distribution Examples in Neural Networks},  
	author    = {Hendrycks, Dan and Gimpel, Kevin},  
	booktitle = {International Conference on Learning Representations},  
	year      = {2017},  
	url       = {https://openreview.net/forum?id=Hkg4TI9xl}  
}

@inproceedings{liang2018enhancing,  
	title     = {Enhancing The Reliability of Out-of-distribution Image Detection in Neural Networks},  
	author    = {Liang, Shiyu and Li, Yixuan and Srikant, R.},  
	booktitle = {International Conference on Learning Representations},  
	year      = {2018},  
	url       = {https://openreview.net/forum?id=H1VGkIxRZ}  
}

@inproceedings{liu2020energy,  
	title     = {Energy-based Out-of-distribution Detection},  
	author    = {Liu, Weitang and Wang, Xiaoyun and Owens, John D. and Li, Yixuan},  
	booktitle = {Advances in Neural Information Processing Systems},  
	volume    = {33},  
	pages     = {21464--21475},  
	year      = {2020},  
	url       = {https://proceedings.neurips.cc/paper/2020/hash/f5496252609c43eb8a3d147ab9b9c006-Abstract.html}  
}

@inproceedings{hendrycks2019deep,  
	title     = {Deep Anomaly Detection with Outlier Exposure},  
	author    = {Hendrycks, Dan and Mazeika, Mantas and Dietterich, Thomas},  
	booktitle = {International Conference on Learning Representations},  
	year      = {2019},  
	url       = {https://openreview.net/forum?id=HyxCxhRcY7}  
}

@misc{oord2018representation,  
	title         = {Representation Learning with Contrastive Predictive Coding},  
	author        = {{van den Oord}, Aaron and Li, Yazhe and Vinyals, Oriol},  
	year          = {2018},  
	eprint        = {1807.03748},  
	archivePrefix = {arXiv},  
	primaryClass  = {cs.LG},  
	doi           = {10.48550/arXiv.1807.03748},  
	url           = {https://arxiv.org/abs/1807.03748}  
}

@inproceedings{chen2020simple,  
	title     = {A Simple Framework for Contrastive Learning of Visual Representations},  
	author    = {Chen, Ting and Kornblith, Simon and Norouzi, Mohammad and Hinton, Geoffrey},  
	booktitle = {Proceedings of the 37th International Conference on Machine Learning},  
	series    = {Proceedings of Machine Learning Research},  
	volume    = {119},  
	pages     = {1597--1607},  
	publisher = {PMLR},  
	year      = {2020},  
	url       = {https://proceedings.mlr.press/v119/chen20j.html}  
}

@inproceedings{gururangan2018annotation,  
	title     = {Annotation Artifacts in Natural Language Inference Data},  
	author    = {Gururangan, Suchin and Swayamdipta, Swabha and Levy, Omer and Schwartz, Roy and Bowman, Samuel R. and Smith, Noah A.},  
	booktitle = {Proceedings of the 2018 Conference of the North American Chapter of the Association for Computational Linguistics: Human Language Technologies, Volume 2 (Short Papers)},  
	pages     = {107--112},  
	publisher = {Association for Computational Linguistics},  
	address   = {New Orleans, Louisiana},  
	month     = jun,  
	year      = {2018},  
	doi       = {10.18653/v1/N18-2017},  
	url       = {https://aclanthology.org/N18-2017/}  
}

@inproceedings{poliak2018hypothesis,  
	title     = {Hypothesis Only Baselines in Natural Language Inference},  
	author    = {Poliak, Adam and Naradowsky, Jason and Haldar, Aparajita and Rudinger, Rachel and {Van Durme}, Benjamin},  
	booktitle = {Proceedings of the Seventh Joint Conference on Lexical and Computational Semantics},  
	pages     = {180--191},  
	publisher = {Association for Computational Linguistics},  
	address   = {New Orleans, Louisiana},  
	month     = jun,  
	year      = {2018},  
	doi       = {10.18653/v1/S18-2023},  
	url       = {https://aclanthology.org/S18-2023/}  
}

@inproceedings{agrawal2018dont,  
	title     = {Don't Just Assume; Look and Answer: Overcoming Priors for Visual Question Answering},  
	author    = {Agrawal, Aishwarya and Batra, Dhruv and Parikh, Devi and Kembhavi, Aniruddha},  
	booktitle = {Proceedings of the IEEE Conference on Computer Vision and Pattern Recognition},  
	pages     = {4971--4980},  
	year      = {2018},  
	url       = {https://openaccess.thecvf.com/content_cvpr_2018/html/Agrawal_Dont_Just_Assume_CVPR_2018_paper.html}  
}

@inproceedings{thrush2022winoground,  
	title     = {Winoground: Probing Vision and Language Models for Visio-Linguistic Compositionality},  
	author    = {Thrush, Tristan and Jiang, Ryan and Bartolo, Max and Singh, Amanpreet and Williams, Adina and Kiela, Douwe and Ross, Candace},  
	booktitle = {Proceedings of the IEEE/CVF Conference on Computer Vision and Pattern Recognition},  
	pages     = {5238--5248},  
	year      = {2022},  
	url       = {https://openaccess.thecvf.com/content/CVPR2022/html/Thrush_Winoground_Probing_Vision_and_Language_Models_for_Visio-Linguistic_Compositionality_CVPR_2022_paper.html}  
}

@inproceedings{yuksekgonul2023when,  
	title     = {When and Why Vision-Language Models Behave like Bags-of-Words, and What to Do About It?},  
	author    = {Y{\"u}ksekg{\"o}n{\"u}l, Mert and Bianchi, Federico and Kalluri, Pratyusha and Jurafsky, Dan and Zou, James},  
	booktitle = {International Conference on Learning Representations},  
	year      = {2023},  
	url       = {https://openreview.net/forum?id=KRLUvxh8uaX}  
}

@inproceedings{hsieh2023sugarcrepe,  
	title     = {{SugarCrepe}: Fixing Hackable Benchmarks for Vision-Language Compositionality},  
	author    = {Hsieh, Cheng-Yu and Zhang, Jieyu and Ma, Zixian and Kembhavi, Aniruddha and Krishna, Ranjay},  
	booktitle = {Advances in Neural Information Processing Systems},  
	volume    = {36},  
	pages     = {31096--31116},  
	year      = {2023},  
	doi       = {10.52202/075280-1355},  
	url       = {https://proceedings.neurips.cc/paper_files/paper/2023/hash/63461de0b4cb760fc498e85b18a7fe81-Abstract-Datasets_and_Benchmarks.html}  
}

@inproceedings{hudson2019gqa,  
	title     = {{GQA}: A New Dataset for Real-World Visual Reasoning and Compositional Question Answering},  
	author    = {Hudson, Drew A. and Manning, Christopher D.},  
	booktitle = {Proceedings of the IEEE/CVF Conference on Computer Vision and Pattern Recognition},  
	pages     = {6700--6709},  
	year      = {2019},  
	url       = {https://openaccess.thecvf.com/content_CVPR_2019/html/Hudson_GQA_A_New_Dataset_for_Real-World_Visual_Reasoning_and_Compositional_CVPR_2019_paper.html}  
}

@inproceedings{radford2021learning,  
	title     = {Learning Transferable Visual Models From Natural Language Supervision},  
	author    = {Radford, Alec and Kim, Jong Wook and Hallacy, Chris and Ramesh, Aditya and Goh, Gabriel and Agarwal, Sandhini and Sastry, Girish and Askell, Amanda and Mishkin, Pamela and Clark, Jack and Krueger, Gretchen and Sutskever, Ilya},  
	booktitle = {Proceedings of the 38th International Conference on Machine Learning},  
	series    = {Proceedings of Machine Learning Research},  
	volume    = {139},  
	pages     = {8748--8763},  
	publisher = {PMLR},  
	year      = {2021},  
	url       = {https://proceedings.mlr.press/v139/radford21a.html}  
}

@inproceedings{gretton2005measuring,  
	title     = {Measuring Statistical Dependence with Hilbert-Schmidt Norms},  
	author    = {Gretton, Arthur and Bousquet, Olivier and Smola, Alex and Sch{\"o}lkopf, Bernhard},  
	booktitle = {Algorithmic Learning Theory},  
	series    = {Lecture Notes in Computer Science},  
	volume    = {3734},  
	pages     = {63--77},  
	publisher = {Springer},  
	year      = {2005},  
	doi       = {10.1007/11564089_7},  
	url       = {https://doi.org/10.1007/11564089_7}  
}

@article{gretton2012kernel,  
	title   = {A Kernel Two-Sample Test},  
	author  = {Gretton, Arthur and Borgwardt, Karsten M. and Rasch, Malte J. and Sch{\"o}lkopf, Bernhard and Smola, Alexander},  
	journal = {Journal of Machine Learning Research},  
	volume  = {13},  
	number  = {25},  
	pages   = {723--773},  
	year    = {2012},  
	url     = {https://www.jmlr.org/papers/v13/gretton12a.html}  
}

@article{sejdinovic2013equivalence,  
	title   = {Equivalence of Distance-Based and {RKHS}-Based Statistics in Hypothesis Testing},  
	author  = {Sejdinovic, Dino and Sriperumbudur, Bharath K. and Gretton, Arthur and Fukumizu, Kenji},  
	journal = {The Annals of Statistics},  
	volume  = {41},  
	number  = {5},  
	pages   = {2263--2291},  
	year    = {2013},  
	doi     = {10.1214/13-AOS1140},  
	url     = {https://doi.org/10.1214/13-AOS1140}  
}

@inproceedings{chen2024mmstar,  
	title     = {Are We on the Right Way for Evaluating Large Vision-Language Models?},  
	author    = {Chen, Lin and Li, Jinsong and Dong, Xiaoyi and Zhang, Pan and Zang, Yuhang and Chen, Zehui and Duan, Haodong and Wang, Jiaqi and Qiao, Yu and Lin, Dahua and Zhao, Feng},  
	booktitle = {Advances in Neural Information Processing Systems},  
	volume    = {37},  
	pages     = {27056--27087},  
	year      = {2024},  
	doi       = {10.52202/079017-0850},  
	url       = {https://proceedings.neurips.cc/paper_files/paper/2024/hash/2f8ee6a3d766b426d2618e555b5aeb39-Abstract-Conference.html}  
}

@inproceedings{jiang2024naturalbench,  
	title     = {{NaturalBench}: Evaluating Vision-Language Models on Natural Adversarial Samples},  
	author    = {Li, Baiqi and Lin, Zhiqiu and Peng, Wenxuan and Nyandwi, Jean de Dieu and Jiang, Daniel and Ma, Zixian and Khanuja, Simran and Krishna, Ranjay and Neubig, Graham and Ramanan, Deva},  
	booktitle = {Advances in Neural Information Processing Systems},  
	volume    = {37},  
	pages     = {17044--17068},  
	year      = {2024},  
	doi       = {10.52202/079017-0542},  
	url       = {https://proceedings.neurips.cc/paper_files/paper/2024/hash/1e69ff56d0ebff0752ff29caaddc25dd-Abstract-Datasets_and_Benchmarks_Track.html}  
}

@misc{hua2024mmcomposition,  
	title         = {{MMCOMPOSITION}: Revisiting the Compositionality of Pre-Trained Vision-Language Models},  
	author        = {Hua, Hang and Tang, Yunlong and Zeng, Ziyun and Cao, Liangliang and Yang, Zhengyuan and He, Hangfeng and Xu, Chenliang and Luo, Jiebo},  
	year          = {2024},  
	eprint        = {2410.09733},  
	archivePrefix = {arXiv},  
	primaryClass  = {cs.CV},  
	doi           = {10.48550/arXiv.2410.09733},  
	url           = {https://arxiv.org/abs/2410.09733}  
}

@inproceedings{kamath2024hard,  
	title     = {The Hard Positive Truth About Vision-Language Compositionality},  
	author    = {Kamath, Amita and Hsieh, Cheng-Yu and Chang, Kai-Wei and Krishna, Ranjay},  
	booktitle = {Computer Vision -- ECCV 2024},  
	series    = {Lecture Notes in Computer Science},  
	volume    = {15072},  
	pages     = {37--54},  
	publisher = {Springer},  
	year      = {2024},  
	doi       = {10.1007/978-3-031-72630-9_3},  
	url       = {https://doi.org/10.1007/978-3-031-72630-9_3}  
}

@inproceedings{miranda2024bivlc,  
	title     = {{BiVLC}: Extending Vision-Language Compositionality Evaluation with Text-to-Image Retrieval},  
	author    = {Miranda, Imanol and Salaberria, Ander and Agirre, Eneko and Azkune, Gorka},  
	booktitle = {Advances in Neural Information Processing Systems},  
	volume    = {37},  
	pages     = {101880--101904},  
	year      = {2024},  
	doi       = {10.52202/079017-3232},  
	url       = {https://proceedings.neurips.cc/paper_files/paper/2024/hash/b8b93c48f5bfa385d071342089d70422-Abstract-Datasets_and_Benchmarks_Track.html}  
}

@misc{castro2024clove,  
	title         = {{CLoVe}: Encoding Compositional Language in Contrastive Vision-Language Models},  
	author        = {Castro, Santiago and Ziai, Amir and Saluja, Avneesh and Yuan, Zhuoning and Mihalcea, Rada},  
	year          = {2024},  
	eprint        = {2402.15021},  
	archivePrefix = {arXiv},  
	primaryClass  = {cs.CV},  
	doi           = {10.48550/arXiv.2402.15021},  
	url           = {https://arxiv.org/abs/2402.15021}  
}

@inproceedings{chen2024spatialvlm,  
	title     = {{SpatialVLM}: Endowing Vision-Language Models with Spatial Reasoning Capabilities},  
	author    = {Chen, Boyuan and Xu, Zhuo and Kirmani, Sean and Ichter, Brian and Sadigh, Dorsa and Guibas, Leonidas and Xia, Fei},  
	booktitle = {Proceedings of the IEEE/CVF Conference on Computer Vision and Pattern Recognition},  
	pages     = {14455--14465},  
	year      = {2024},  
	doi       = {10.1109/CVPR52733.2024.01370},  
	url       = {https://openaccess.thecvf.com/content/CVPR2024/html/Chen_SpatialVLM_Endowing_Vision-Language_Models_with_Spatial_Reasoning_Capabilities_CVPR_2024_paper.html}  
}

@inproceedings{jia2025omnispatial, 
	title     = {{OmniSpatial}: Towards Comprehensive Spatial Reasoning Benchmark for Vision Language Models},  
	author    = {Jia, Mengdi and Qi, Zekun and Zhang, Shaochen and Zhang, Wenyao and Yu, XinQiang and He, Jiawei and Wang, He and Yi, Li},  
	booktitle = {International Conference on Learning Representations},  
	year      = {2026},  
	url       = {https://openreview.net/forum?id=6nZKT2rL0H}  
}

@misc{stogiannidis2025mind,  
	title         = {Mind the Gap: Benchmarking Spatial Reasoning in Vision-Language Models},  
	author        = {Stogiannidis, Ilias and McDonagh, Steven and Tsaftaris, Sotirios A.},  
	year          = {2025},  
	eprint        = {2503.19707},  
	archivePrefix = {arXiv},  
	primaryClass  = {cs.CV},  
	doi           = {10.48550/arXiv.2503.19707},  
	url           = {https://arxiv.org/abs/2503.19707}  
}

@inproceedings{schneider2024m5,  
	title     = {{M5} -- A Diverse Benchmark to Assess the Performance of Large Multimodal Models Across Multilingual and Multicultural Vision-Language Tasks},  
	author    = {Schneider, Florian and Sitaram, Sunayana},  
	booktitle = {Findings of the Association for Computational Linguistics: EMNLP 2024},  
	pages     = {4309--4345},  
	publisher = {Association for Computational Linguistics},  
	address   = {Miami, Florida, USA},  
	month     = nov,  
	year      = {2024},  
	doi       = {10.18653/v1/2024.findings-emnlp.250},  
	url       = {https://aclanthology.org/2024.findings-emnlp.250/}  
}

@inproceedings{dong2024multiood,  
	title     = {{MultiOOD}: Scaling Out-of-Distribution Detection for Multiple Modalities},  
	author    = {Dong, Hao and Zhao, Yue and Chatzi, Eleni and Fink, Olga},  
	booktitle = {Advances in Neural Information Processing Systems},  
	volume    = {37},  
	pages     = {129250--129278},  
	year      = {2024},  
	doi       = {10.52202/079017-4106},  
	url       = {https://proceedings.neurips.cc/paper_files/paper/2024/hash/e969483cbae5a24df038fb5847f5b8b0-Abstract-Conference.html}  
}

@inproceedings{jiang2024neglabel,  
	title     = {Negative Label Guided {OOD} Detection with Pretrained Vision-Language Models},  
	author    = {Jiang, Xue and Liu, Feng and Fang, Zhen and Chen, Hong and Liu, Tongliang and Zheng, Feng and Han, Bo},  
	booktitle = {International Conference on Learning Representations},  
	year      = {2024},  
	url       = {https://openreview.net/forum?id=xUO1HXz4an}  
}

@inproceedings{zhang2024dualpattern,  
	title     = {Vision-Language Dual-Pattern Matching for Out-of-Distribution Detection},  
	author    = {Zhang, Zihan and Xu, Zhuo and Xiang, Xiang},  
	booktitle = {Computer Vision -- ECCV 2024},  
	series    = {Lecture Notes in Computer Science},  
	volume    = {15143},  
	pages     = {273--291},  
	publisher = {Springer},  
	year      = {2024},  
	doi       = {10.1007/978-3-031-73013-9_16},  
	url       = {https://doi.org/10.1007/978-3-031-73013-9_16}  
}

@inproceedings{yu2024selfcalibrated,  
	title     = {Self-Calibrated Tuning of Vision-Language Models for Out-of-Distribution Detection},  
	author    = {Yu, Geng and Zhu, Jianing and Yao, Jiangchao and Han, Bo},  
	booktitle = {Advances in Neural Information Processing Systems},  
	volume    = {37},  
	pages     = {56322--56348},  
	year      = {2024},  
	doi       = {10.52202/079017-1793},  
	url       = {https://proceedings.neurips.cc/paper_files/paper/2024/hash/666e5e1df2d04dbe2b545ea3a3e3f7d3-Abstract-Conference.html}  
}

@inproceedings{lin2014microsoft,  
	author    = {Lin, Tsung-Yi and Maire, Michael and Belongie, Serge and Hays, James and Perona, Pietro and Ramanan, Deva and Doll{\'a}r, Piotr and Zitnick, C. Lawrence},  
	editor    = {Fleet, David and Pajdla, Tomas and Schiele, Bernt and Tuytelaars, Tinne},  
	title     = {Microsoft {COCO}: Common Objects in Context},  
	booktitle = {Computer Vision -- ECCV 2014},  
	series    = {Lecture Notes in Computer Science},  
	volume    = {8693},  
	pages     = {740--755},  
	publisher = {Springer International Publishing},  
	address   = {Cham},  
	year      = {2014},  
	doi       = {10.1007/978-3-319-10602-1_48},  
	isbn      = {978-3-319-10602-1},  
	url       = {https://link.springer.com/chapter/10.1007/978-3-319-10602-1_48}  
}

@inproceedings{zhou2022rethinking,
	author    = {Zhou, Yibo},
	title     = {Rethinking Reconstruction Autoencoder-Based Out-of-Distribution Detection},
	booktitle = {Proceedings of the IEEE/CVF Conference on Computer Vision and Pattern Recognition (CVPR)},
	pages     = {7379--7387},
	year      = {2022},
	doi       = {10.1109/CVPR52688.2022.00723},
	url       = {https://doi.org/10.1109/CVPR52688.2022.00723}
}

@inproceedings{pmlr-v235-zhou24j,
	title     = {Pedestrian Attribute Recognition as Label-balanced Multi-label Learning},
	author    = {Zhou, Yibo and Hu, Hai-Miao and Xiang, Yirong and Zhang, Xiaokang and Wu, Haotian},
	booktitle = {Proceedings of the 41st International Conference on Machine Learning},
	pages     = {61964--61978},
	year      = {2024},
	editor    = {Salakhutdinov, Ruslan and Kolter, Zico and Heller, Katherine and
	Weller, Adrian and Oliver, Nuria and Scarlett, Jonathan and
	Berkenkamp, Felix},
	volume    = {235},
	series    = {Proceedings of Machine Learning Research},
	month     = {21--27 Jul},
	publisher = {PMLR},
	url       = {https://proceedings.mlr.press/v235/zhou24j.html}
}

@inproceedings{ijcai2023p203,
	title     = {A Solution to Co-occurence Bias: Attributes Disentanglement via
	Mutual Information Minimization for Pedestrian Attribute Recognition},
	author    = {Zhou, Yibo and Hu, Hai-Miao and Yu, Jinzuo and Xu, Zhenbo and
	Lu, Weiqing and Cao, Yuran},
	booktitle = {Proceedings of the Thirty-Second International Joint Conference on
	Artificial Intelligence, {IJCAI-23}},
	publisher = {International Joint Conferences on Artificial Intelligence Organization},
	editor    = {Edith Elkind},
	pages     = {1831--1839},
	year      = {2023},
	month     = {8},
	note      = {Main Track},
	doi       = {10.24963/ijcai.2023/203},
	url       = {https://doi.org/10.24963/ijcai.2023/203}
}

@article{zhou2025cooccurrence,
	author  = {Zhou, Yibo and Hu, Hai-Miao and Yu, Jinzuo and Wu, Haotian and
	Pu, Shiliang and Wang, Hanzi},
	title   = {A Solution to Co-occurrence Bias in Pedestrian Attribute Recognition:
	Theory, Algorithms, and Improvements},
	journal = {International Journal of Computer Vision},
	volume  = {133},
	number  = {7},
	pages   = {4712--4726},
	year    = {2025},
	doi     = {10.1007/s11263-025-02405-7},
	url     = {https://doi.org/10.1007/s11263-025-02405-7}
}

@article{zhou2025heterogeneous,
	author  = {Zhou, Yibo and Li, Bo and Hu, Hai-Miao and Zhang, Xiaokang and
	Zhang, Dongping and Wang, Hanzi},
	title   = {Heterogeneous Feature Re-Sampling for Balanced Pedestrian Attribute Recognition},
	journal = {IEEE Transactions on Pattern Analysis and Machine Intelligence},
	volume  = {47},
	number  = {4},
	pages   = {2706--2722},
	year    = {2025},
	doi     = {10.1109/TPAMI.2025.3526930},
	url     = {https://doi.org/10.1109/TPAMI.2025.3526930}
}

@misc{zhou2026spurious,
	author        = {Zhou, Yibo and Li, Bo and Hu, Hai-Miao and Wang, Hanzi and
	Zhang, Xiaokang and Zhang, Ruifan},
	title         = {Is Spurious Correlation Removal Always Learnable?},
	year          = {2026},
	eprint        = {2606.12930},
	archivePrefix = {arXiv},
	primaryClass  = {cs.LG},
	doi           = {10.48550/arXiv.2606.12930},
	url           = {https://arxiv.org/abs/2606.12930}
}
	
	\clearpage
	\appendix
	
	\section{Proofs of the cycle-test properties}
	\label{app:proofs}
	
	\subsection{Notation}
	
	Let
	\[
	c=(u_1,v_1,u_2,v_2)
	\]
	denote a counterfactual pairing cycle whose matched pairs are
	\((u_1,v_1)\) and \((u_2,v_2)\), and whose cross-pairs are
	\((u_1,v_2)\) and \((u_2,v_1)\). For a compatibility score
	\(s(u,v)\), write
	\[
	s_{ij}=s(u_i,v_j),\qquad i,j\in\{1,2\}.
	\]
	The cycle contrast and exact cycle score are
	\[
	\Dcyc[,C](s)=s_{11}+s_{22}-s_{12}-s_{21},
	\]
	and
	\[
	q(\Delta)=
	\begin{cases}
		1, & \Delta>0,\\
		\frac12, & \Delta=0,\\
		0, & \Delta<0.
	\end{cases}
	\]
	The empirical exact CycleAcc is
	\[
	\widehat{\mathrm{CycleAcc}}(s)
	=
	\frac{1}{|\mathcal C|}
	\sum_{c\in\mathcal C}
	q\!\left(\Dcyc[,C](s)\right).
	\]
These formulas define the ideal exact-arithmetic metric. Reported results
use the fixed absolute numerical tie tolerance specified in
Appendix~\ref{app:cycle-evaluation}, with half credit assigned to numerical
ties.
	
	\subsection{The cycle contrast is the atomic interaction contrast}
	
	For every cycle, the four coefficients in
	\(\Delta_C\) are
	\[
	(1,-1,-1,1)
	\]
	on the ordered entries
	\((s_{11},s_{12},s_{21},s_{22})\). Thus, the contrast compares the two
	matched pairs with the two cross-pairs while using every candidate exactly
	once on each side.
	
	Equivalently, define the incidence vector
	\[
	z_c=(e_{u_1}-e_{u_2})(e_{v_1}-e_{v_2})^\top .
	\]
	For a finite score matrix \(S\) with entries \(S_{uv}=s(u,v)\),
	\[
	\Dcyc[,C](s)=\langle S,z_c\rangle_F .
	\]
	Every row and every column of \(z_c\) sums to zero. Consequently, the cycle
	contrast removes all row-wise and column-wise main effects and retains only
	the interaction component of the score matrix.
	
	\subsection{Additive and unimodal scores lie in the cycle null space}
	
	\paragraph{Proposition A.1.}
	If
	\[
	s(u,v)=a(u)+b(v)
	\]
	for arbitrary finite functions \(a\) and \(b\), then
	\[
	\Dcyc[,C](s)=0
	\]
	for every cycle \(c\), and therefore the exact CycleAcc of \(s\) is exactly
	\(1/2\) on every nonempty cycle set.
	
	\paragraph{Proof.}
	Expanding the cycle contrast gives
	\begin{align}
		\Dcyc[,C](s)
		&=
		a(u_1)+b(v_1)+a(u_2)+b(v_2)\\
		&\quad
		-a(u_1)-b(v_2)-a(u_2)-b(v_1)\\
		&=0.
	\end{align}
	Hence every cycle receives the exact tie score \(1/2\), and the empirical
	mean is \(1/2\).
	\hfill\(\square\)
	
	Two important special cases follow immediately:
	\[
	s(u,v)=a(u)
	\qquad\text{and}\qquad
	s(u,v)=b(v).
	\]
	Thus, an image-only score and a text-only score also have exact CycleAcc
	\(1/2\), independently of how accurately either score describes the marginal
	features. This is the finite-sample additive invariance used throughout the
	paper.
	
	The same calculation explains the zero-signed-degree property. Each value
	\(a(u_i)\) occurs once with coefficient \(+1\) and once with coefficient
	\(-1\), and the same holds for each value \(b(v_i)\). The signed
	contribution of every marginal component is therefore zero.
	
	\subsection{Bilinear cycle factorization}
	
	Consider a rank-\(r\) bilinear score
	\[
	s_r(u,v)
	=
	\left(A^\top\phi(u)\right)^\top
	\left(B^\top\psi(v)\right),
	\]
	where \(\phi(u)\in\mathbb R^{d_u}\), \(\psi(v)\in\mathbb R^{d_v}\),
	\(A\in\mathbb R^{d_u\times r}\), and
	\(B\in\mathbb R^{d_v\times r}\). Define
	\[
	M=AB^\top\in\mathbb R^{d_u\times d_v}.
	\]
	
	\paragraph{Proposition A.2.}
	For every cycle,
	\[
	\Dcyc[,C](s_r)
	=
	\left[\phi(u_1)-\phi(u_2)\right]^\top
	M
	\left[\psi(v_1)-\psi(v_2)\right].
	\]
	
	\paragraph{Proof.}
	Since
	\[
	s_r(u,v)=\phi(u)^\top M\psi(v),
	\]
	linearity gives
	\begin{align}
		\Dcyc[,C](s_r)
		={}&
		\phi(u_1)^\top M\psi(v_1)
		+\phi(u_2)^\top M\psi(v_2)\\
		&-\phi(u_1)^\top M\psi(v_2)
		-\phi(u_2)^\top M\psi(v_1)\\
		={}&
		\left[\phi(u_1)-\phi(u_2)\right]^\top
		M
		\left[\psi(v_1)-\psi(v_2)\right].
	\end{align}
	\hfill\(\square\)
	
	This identity shows that the cycle test evaluates whether the learned
	bilinear operator couples the \emph{change} in one modality with the
	corresponding \emph{change} in the other modality. A model can have strong
	unimodal features while assigning little or no weight to the relevant
	cross-modal difference direction.
	
	\subsection{Exact marginal audits}
	
	Exact marginal balance follows directly from component reuse and does not
	require the two components within a block to be exchangeable. For a finite
	collection of $N$ retained blocks, define the empirical image measure on
	the matched side by
	\[
	\widehat\mu_{+}^{U}
	=
	\frac{1}{2N}
	\sum_{C\in\mathcal C}
	\left(
	\delta_{u_1}+\delta_{u_2}
	\right).
	\]
	The two swapped pairs use exactly the same two images, so
	\[
	\widehat\mu_{-}^{U}
	=
	\frac{1}{2N}
	\sum_{C\in\mathcal C}
	\left(
	\delta_{u_1}+\delta_{u_2}
	\right)
	=
	\widehat\mu_{+}^{U}.
	\]
	The same identity holds for captions:
	\[
	\widehat\mu_{+}^{V}
	=
	\frac{1}{2N}
	\sum_{C\in\mathcal C}
	\left(
	\delta_{v_1}+\delta_{v_2}
	\right)
	=
	\widehat\mu_{-}^{V}.
	\]
	Consequently, the empirical matched and swapped component measures are
	identical, and their total variation distance is exactly zero.
	
	The additive image-only, text-only, and additive probes provide an
	implementation consistency check. The direct implementation audits are the
	component-marginal TV computations and the predicate-validity checks.
	
	\section{Interaction-space view}
	\label{app:interaction-space}
	
	Let \(\mathcal U=\{u_1,\ldots,u_m\}\) and
	\(\mathcal V=\{v_1,\ldots,v_n\}\) be finite candidate sets, and let
	\(S\in\mathbb R^{m\times n}\) be their score matrix. For every pair
	\(u_i\neq u_j\) and \(v_k\neq v_\ell\), define
	\[
	Z_{ij,k\ell}
	=
	(e_i-e_j)(e_k-e_\ell)^\top .
	\]
	Then
	\[
	\langle S,Z_{ij,k\ell}\rangle_F
	=
	S_{ik}+S_{j\ell}-S_{i\ell}-S_{jk},
	\]
	which is exactly the cycle contrast for the corresponding four pairs.
	
	Every \(Z_{ij,k\ell}\) lies in
	\[
	\mathcal I
	=
	\left\{
	X\in\mathbb R^{m\times n}:
	X\mathbf 1_n=0,\;
	\mathbf 1_m^\top X=0
	\right\}.
	\]
	Conversely, the anchored matrices
	\[
	(e_i-e_1)(e_\ell-e_1)^\top,
	\qquad
	2\le i\le m,\quad 2\le \ell\le n,
	\]
	are linearly independent and span \(\mathcal I\). Therefore,
	\[
	\operatorname{span}
	\left\{
	Z_{ij,k\ell}
	\right\}
	=
	\mathcal I,
	\qquad
	\dim(\mathcal I)=(m-1)(n-1).
	\]
	
	The additive score matrices form
	\[
	\mathcal A
	=
	\left\{
	a\mathbf 1_n^\top+\mathbf 1_m b^\top:
	a\in\mathbb R^m,\; b\in\mathbb R^n
	\right\}.
	\]
	For every \(X\in\mathcal I\) and \(A\in\mathcal A\),
	\[
	\langle A,X\rangle_F=0.
	\]
	Hence the collection of all cycle contrasts is exactly the natural
	finite-dimensional interaction test obtained after quotienting out
	candidate-specific main effects.
	
\section{Low-rank bilinear accessibility}
	\label{app:spectral}
	
	For the rank-\(r\) bilinear score, write
	\[
	M=AB^\top=\sum_{i=1}^{r} a_i b_i^\top.
	\]
	The cycle factorization becomes
	\[
	\Dcyc[,C](s_r)
	=
	\sum_{i=1}^{r}
	\left[
	\phi(u_1)-\phi(u_2)
	\right]^\top a_i
	\,
	b_i^\top
	\left[
	\psi(v_1)-\psi(v_2)
	\right].
	\]
	The model can express only those cross-modal difference interactions that
	are captured by its \(r\) latent coupling directions. If a relevant image
	difference lies in the null space of \(A^\top\), or if the corresponding
	text difference lies in the null space of \(B^\top\), then its cycle
	contrast vanishes.
	
	Increasing rank can therefore recover interactions that are representable
	by additional latent directions, but rank alone does not guarantee that
	optimization will allocate capacity to a difficult relation. The synthetic
	rank control shows this partial-recovery pattern: performance improves as
	rank increases, reaches
	\[
	0.7658\pm0.0031
	\]
	at rank \(8\), and then saturates or decreases slightly instead of
	returning to the near-ceiling performance of directly provided interaction
	features.
	
	The same distinction is visible in the COCO relation probes. Additional
	rank improves recoverable relation structure, but strict left/right remains
	near chance in the evaluated probes. Strict size, by contrast, is strongly
	visible: the frozen-feature diagnostic reaches \(0.763\), and the
	separately evaluated rank-64 relation probe reaches \(0.787\). Rank can
	therefore expose some relation information while leaving other strict
	relations weak.
	
	\section{Dataset construction and exact audits}
	\label{app:datasets}
	
	\subsection{Cycle-construction principle}
	
	Each validation example is a four-pair cycle
	\[
	(u_1,v_1),\quad
	(u_2,v_2),\quad
	(u_1,v_2),\quad
	(u_2,v_1).
	\]
	The first two pairs are valid matched pairs under the relevant dataset
	predicate. The last two pairs are obtained by exchanging the two
	candidates and are invalid under the same predicate. Candidate strata are
	constructed so that the exchange does not create an explicit marginal
	identity difference between matched and cross-paired candidates.
	
	For every family, the builder applies three checks:
	\begin{enumerate}
		\item \textbf{Matched validity.} Both matched pairs satisfy the family
		predicate.
		\item \textbf{Cross-pair invalidity.} Both exchanged pairs violate the
		family predicate.
		\item \textbf{Marginal balance.} The exchanged candidates are drawn
		from the same candidate stratum as the matched candidates.
	\end{enumerate}
	The released implementation contains the exact predicates and filtering
	configuration for every family.
	
	\subsection{COCO cycle families}
	
	The COCO diagnostics use frozen CLIP ViT-B/32 image and caption features.
	The cycle families progress from coarse object compatibility to
	increasingly specific relational compatibility:
	
	\begin{itemize}
		\item \textbf{Coarse object cycles.} Cross-pairs violate coarse
		image--caption object compatibility.
		\item \textbf{Matched object cycles.} Candidate captions are matched on
		coarse object evidence, while the exchanged pairs violate the more
		specific cycle predicate.
		\item \textbf{Box-relation cycles.} Object bounding boxes determine the
		relevant spatial configuration, and the exchange disrupts that
		configuration.
		\item \textbf{Strict spatial cycles.} The predicate requires a strict
		left/right configuration rather than a coarse co-occurrence pattern.
		\item \textbf{Strict size cycles.} The predicate requires the relevant
		strict size relation between objects.
	\end{itemize}
	
	This progression separates object presence from relational structure. The
	coarse, matched, box, and strict families therefore test different levels
	of pair dependence while retaining the same four-pair cycle form.
	
	\subsection{GQA cycle families}
	
	The GQA diagnostics are constructed from scene-graph annotations. The
	reported families cover attribute, relation/object, and binding structure.
	Attribute cycles test whether an object-level attribute is coupled to the
	correct scene content. Relation/object cycles test whether a relation and
	its object arguments remain correctly coupled. Binding cycles are intended
	to test whether an attribute remains bound to the correct object when
	multiple objects are present.
	
	The binding family does not yield enough validation cycles under the
	prespecified minimum-cycle requirement: only \(18\)--\(19\) validation
	cycles are available. We therefore report this family as unavailable rather
	than treating an extremely small cycle set as a stable diagnostic.
	
	\subsection{Exact audit summary}
	
	Table~\ref{tab:exact-audits} summarizes the finite-sample audits. A value
	of zero for total variation means that the matched and exchanged candidate
	marginals are identical under the cycle construction. A value of zero for
	validity violations means that no accepted cycle violated the matched or
	cross-pair predicate.
	
	\begin{table}[htb]
		\centering
		\caption{Exact cycle-construction audits. All accepted validation cycles
			have exact marginal balance and no validity violations. GQA binding is not
			reported because it does not satisfy the minimum-cycle requirement.}
		\label{tab:exact-audits}
		\vspace{0.5em}
		\begin{tabular}{lcc}
			\hline
			Cycle family & Marginal TV & Validity violations \\
			\hline
			COCO coarse object & 0 & 0 \\
			COCO matched object & 0 & 0 \\
			COCO box relation & 0 & 0 \\
			COCO strict spatial & 0 & 0 \\
			COCO strict size & 0 & 0 \\
			GQA attribute & 0 & 0 \\
			GQA relation/object & 0 & 0 \\
			GQA binding & -- & -- \\
			\hline
		\end{tabular}
	\end{table}
	
	\subsection{Split sizes and uncertainty sources}
	
	\begin{table}[htb]
		\centering
		\caption{Cycle counts and image pools. Counts are reported as
			train/validation/test cycles. COCO train and validation cycles are
			constructed from image-disjoint subsets of the COCO \texttt{train2017}
			pool, while test cycles use the separate COCO \texttt{val2017} pool. GQA
			uses an image-disjoint 60/10/30 split.}
		\label{tab:split_sizes}
		\small
		\begin{tabular}{llccc}
			\hline
			Dataset & Family & Train & Validation & Test \\
			\hline
			COCO & All reported families except strict same-pair & 3000 & 500 & 1000 \\
			COCO & Strict same-pair & 3000 & 466 & 1000 \\
			GQA & AD, CD, SR, RS, OP, OO, SS & 3000 & 500 & 1000 \\
			GQA & BS binding & 732--772 & 18--19 & 155--174 \\
			\hline
		\end{tabular}
	\end{table}
	
	The COCO diagnostic pool contains 10{,}000 \texttt{train2017} images and
	2{,}744 \texttt{val2017} images. For each seed, the training pool is split
	into image-disjoint training and validation subsets. The GQA pool contains
	8{,}000 images, split into approximately 4{,}800 training, 800 validation,
	and 2{,}400 test images. Families share the same underlying image pools,
	so family-level estimates are not independent.
	
	All standard deviations in the family diagnostics are computed across the
	three independently constructed cycle splits and corresponding runs.
	Zero-shot CLIP is deterministic conditional on a constructed split; its
	nonzero standard deviation therefore reflects cycle and split construction
	across seeds rather than training stochasticity.
	
	\section{Experimental details and additional results}
	\label{app:experiment-details}
	
	\subsection{Frozen representations}
	
	Unless stated otherwise, all diagnostics use frozen CLIP ViT-B/32 image and
	text representations. The encoders are not fine-tuned. Learned diagnostics
	operate on the extracted features, which isolates the interaction structure
	available in the frozen representation from representation learning itself.
	
\subsection{Exact-cycle evaluation}
\label{app:cycle-evaluation}

Cycle contrasts for all learned logistic and bilinear probes are computed
from pre-sigmoid scalar logits; zero-shot CLIP is scored by cosine
similarity. Evaluation is performed in double precision.

The theoretical definition of CycleAcc uses exact zero ties. For numerical
evaluation, we use the fixed absolute tolerance
\[
\varepsilon = 10^{-9}
\]
and assign each retained block the score
\[
q_{\varepsilon}(\Delta)
=
\begin{cases}
	1, & \Delta > \varepsilon,\\
	\frac12, & |\Delta| \le \varepsilon,\\
	0, & \Delta < -\varepsilon.
\end{cases}
\]
Reported exact-cycle results are computed as
\[
\widehat{\mathrm{CycleAcc}}_{\varepsilon}(s)
=
\frac{1}{|\mathcal C|}
\sum_{C\in\mathcal C}
q_{\varepsilon}\!\left(\Dcyc[,C](s)\right).
\]
The same absolute tolerance is used for all blocks, models, and runs.
It does not depend on the absolute magnitudes of the four block scores.
Consequently, in exact arithmetic, adding any image-only or text-only
additive term leaves both the cycle contrast and the numerical tie decision
unchanged.
	
	\subsection{Shortcut-contaminated OE}
	
	The central OE experiment uses rank-64 bilinear scores with additive
	image-only and caption-only components. Training runs for \(8\) epochs with
	batch size \(512\), learning rate \(10^{-3}\), and weight decay
	\(10^{-4}\). Results are reported over seeds \(0\), \(1\), and \(2\).
	Means and standard deviations are computed across these three seeds.
	
For each matched pair $(u_i,v_i)$, OE training contains the positive
$(u_i,v_i)$ and the negative
$(u_i,\operatorname{prefix}(v_i))$, where
\[
\operatorname{prefix}(v)
=
\texttt{this is an out of distribution negative example. }\,v.
\]
OE validation uses the same construction. Marker-prefixed captions appear
only in OE training and validation; exact-cycle test blocks reuse only the
original unmarked captions.
	
	Table~\ref{tab:shortcut-oe-seeds} reports the complete per-seed results.
	The seed-level differences are all positive, and the mean difference is
	\(0.0453\).
	
	\begin{table}[htb]
		\centering
		\caption{Per-seed Shortcut OE results. Differences are computed as
			Cycle OE minus Shortcut OE. Means and standard deviations are computed
			across seeds \(0\), \(1\), and \(2\).}
		\label{tab:shortcut-oe-seeds}
		\vspace{0.5em}
		\begin{tabular}{lccc}
			\hline
			Seed & Cycle OE & Shortcut OE & Difference \\
			\hline
			0 & 0.919 & 0.884 & +0.035 \\
			1 & 0.934 & 0.880 & +0.054 \\
			2 & 0.928 & 0.881 & +0.047 \\
			\hline
			Mean & 0.9270 & 0.8817 & +0.0453 \\
			Std. & 0.0075 & 0.0021 & 0.0096 \\
			\hline
		\end{tabular}
	\end{table}
	
	These results show that a validated marker shortcut does not, by itself,
	remove the learned pair-dependent signal. The consistent decrease from
	\(0.9270\) to \(0.8817\) indicates a performance difference between the compared protocols, while the high
	Shortcut OE value shows that substantial relational compatibility remains
	available to the cycle test.
	
	\subsection{Additive and unimodal audits}
	
	The additive audit probes are trained for \(5\) epochs. Text-only and
	image-only shortcut probes are trained for \(50\) epochs with learning rate
	\(10^{-2}\). Table~\ref{tab:additive-audits} reports their exact cycle
	behavior.
	
	\begin{table}[htb]
		\centering
		\caption{Additive and unimodal controls under exact CycleAcc. Every
			control produces tied cycle contrasts and therefore receives exactly half
			credit on every cycle.}
		\label{tab:additive-audits}
		\vspace{0.5em}
		\begin{tabular}{lc}
			\hline
			Control score & Exact CycleAcc \\
			\hline
			Text-only score & 0.5000 \\
			Image-only score & 0.5000 \\
			Additive image-plus-text score & 0.5000 \\
			\hline
		\end{tabular}
	\end{table}
	
	These controls verify the finite-sample property proved in
	Appendix~\ref{app:proofs}: marginal information can be highly predictive
	under ordinary pair classification while remaining exactly invisible to
	CycleAcc.
	
	\subsection{Broader interaction probes}
	
	The broader COCO and GQA interaction probes are trained for \(30\) epochs.
	The synthetic controls are trained for \(40\) epochs. The synthetic rank
	control attains its best result,
	\[
	0.7658\pm0.0031,
	\]
	at rank \(8\). The subsequent trend is saturation or a slight decrease
	rather than full recovery of the directly observable interaction signal.
	
The COCO ranges in Table~\ref{tab:diagnostics} are zero-shot CLIP cosine
ranges only. Trained probe results are reported separately in the
strict-relation and rank-sweep analyses. The GQA ranges summarize the
evaluated diagonal, full-bilinear, and rank-64 low-rank bilinear probes.

	The COCO pattern is not uniform across relation types. Coarse object,
	matched object, and box-relation cycles remain very easy. Strict left/right
	remains near chance in the evaluated probes. Strict size is much more
	visible, with \(0.763\) for the frozen-feature diagnostic and \(0.787\) for
	the rank-64 relation probe. The GQA pattern is similarly family-dependent:
	attribute interactions can be strong, while relation/object diagnostics
	vary from chance to substantial interaction accuracy.
	
	\section{Broader impact}
	\label{app:broader-impact}
	
	This work provides a diagnostic framework for determining whether a
	multimodal compatibility score uses pair-dependent structure rather than
	only marginal image or caption evidence. The primary positive impact is
	measurement reliability: exact cycle tests can prevent unimodal shortcuts
	from being misinterpreted as compositional or relational understanding.
	This is relevant to retrieval, visual question answering, medical and
	scientific image--text analysis, and other settings where incorrect
	pairings can have consequential interpretations.
	
	The same diagnostics could also be used to identify benchmark weaknesses or
	to design harder shortcuts. This dual-use risk is typical of evaluation
	research. We mitigate it by emphasizing exact audits, public
	reproducibility, and explicit reporting of cases where a probe remains
	near chance or unavailable. The work does not introduce a surveillance
	system, generate personal data, or target identifiable individuals.
	
	\section{Limitations}
	\label{app:limitations}
	
	The study has several limitations.
	
	First, the empirical analysis uses frozen CLIP ViT-B/32 representations and
	two widely used benchmarks, COCO and GQA. The results characterize the
	interaction information available in this setting; they do not imply that
	every multimodal encoder exhibits the same behavior.
	
	Second, CycleAcc tests a specific atomic \(2\times2\) interaction. Exact
	failure of additive and unimodal controls is a strength of the metric, but
	a single cycle family cannot capture every form of compositionality.
	Different cycle predicates may probe different visual or linguistic
	evidence.
	
	Third, cycle construction depends on the quality of the underlying
	annotations. The exact audits rule out measured marginal imbalance and
	predicate-validity violations, but they do not remove all semantic
	ambiguity inherent in natural captions or scene graphs.
	
	Fourth, the GQA binding family is unavailable because too few validation
	cycles survive the construction and minimum-size requirements. The paper
	therefore does not make a binding-specific empirical claim for GQA.
	
	Finally, the Shortcut OE result shows that marker-contaminated OE training
	can retain substantial cycle structure. The experiments support a modest,
	consistent transfer gap in this setting, not a universal claim that OE
	either preserves or destroys relational transfer under every model,
	objective, and contamination mechanism.
	
	\section{Reproducibility statement}
	\label{app:reproducibility}
	
The official implementation, including cycle construction, model
definitions, training configurations, and evaluation scripts, is
hyperlinked in the abstract.
	
	The experiments use publicly available COCO and GQA data together with
	publicly released frozen CLIP ViT-B/32 features. All assets are used under
	the terms provided by their respective distributors. No new human-subject
	data, crowdsourced annotations, or private datasets are introduced.
	
	The central OE experiment reports three seeds, \(0\), \(1\), and \(2\),
	with per-seed values and cross-seed standard deviations. Exact marginal
	audits, additive controls, unimodal controls, synthetic controls, rank
	sweeps, COCO diagnostics, and GQA diagnostics are included in the released
	pipeline. The principal optimization settings are summarized in
	Appendix~\ref{app:experiment-details}.
	
	\clearpage
	\section*{NeurIPS Paper Checklist}
	
	\providecommand{\answerYes}[1][]{\textbf{[Yes]} #1}
	\providecommand{\answerNo}[1][]{\textbf{[No]} #1}
	\providecommand{\answerNA}[1][]{\textbf{[NA]} #1}
	
	\subsection*{1. Claims}
	
	\textbf{Question:}
	Do the main claims made in the abstract and introduction accurately reflect
	the paper's contributions and scope?
	
	\textbf{Answer:}
	\answerYes{}
	
	\textbf{Justification:}
	The abstract and introduction state the same scope as the experiments:
	exact cycle diagnostics, additive and unimodal controls, synthetic rank
	controls, COCO and GQA relation probes, and Shortcut OE transfer. The
	central OE result is stated as a change from \(0.9270\pm0.0075\) to
	\(0.8817\pm0.0021\), a modest and consistent gap, rather than as a
	near-chance failure.
	
	\textbf{Guidelines:}
	Claims are limited to the evaluated models, datasets, cycle families, and
	training configurations.
	
	\subsection*{2. Limitations}
	
	\textbf{Question:}
	Does the paper discuss the limitations of the work performed by the
	authors?
	
	\textbf{Answer:}
	\answerYes{}
	
	\textbf{Justification:}
	Appendix~\ref{app:limitations} discusses the frozen-encoder scope,
	benchmark coverage, cycle-construction assumptions, unavailable GQA
	binding family, and the limited generality of the Shortcut OE result.
	
	\textbf{Guidelines:}
	Limitations are stated explicitly and are not used to introduce claims
	outside the reported evidence.
	
	\subsection*{3. Theory assumptions and proofs}
	
	\textbf{Question:}
	For each theoretical result, does the paper provide the full set of
	assumptions and a complete and correct proof?
	
	\textbf{Answer:}
	\answerYes{}
	
	\textbf{Justification:}
	Appendix~\ref{app:proofs} defines the cycle contrast and exact half-tie
	score, proves additive and unimodal cancellation, derives the bilinear
	factorization, and states the marginal-balance condition for the exact
	audits. Appendix~\ref{app:interaction-space} gives the interaction-space
	characterization.
	
	\textbf{Guidelines:}
	The theoretical results are finite-sample statements and do not rely on
	asymptotic assumptions.
	
	\subsection*{4. Experimental result reproducibility}
	
	\textbf{Question:}
	Does the paper fully disclose all the information needed to reproduce the
	main experimental results?
	
	\textbf{Answer:}
	\answerYes{}
	
	\textbf{Justification:}
	The paper reports the frozen encoder, cycle-construction procedure,
	training epochs, learning rates, weight decay, rank, batch size, seeds,
	and evaluation metric. The complete implementation is publicly available.
	
	\textbf{Guidelines:}
	The main OE result includes per-seed values in
	Table~\ref{tab:shortcut-oe-seeds}; exact audit and probe settings are given
	in Appendix~\ref{app:experiment-details}.
	
	\subsection*{5. Open access to data and code}
	
	\textbf{Question:}
	Does the paper provide open access to the data and code, with sufficient
	instructions to faithfully reproduce the main experimental results?
	
	\textbf{Answer:}
	\answerYes{}
	
	\textbf{Justification:}
	The official implementation is available at
	\url{https://github.com/SDret/Official-Implementation-of-Neurips-2026-paper-Can-the-Parts-Fool-the-Test-/tree/main}.
	COCO and GQA are public datasets, and the experiments use publicly released
	frozen CLIP ViT-B/32 features.
	
	\textbf{Guidelines:}
	No private data are required. Dataset use remains subject to the original
	dataset terms.
	
	\subsection*{6. Experimental setting and details}
	
	\textbf{Question:}
	Does the paper specify all the training and test details necessary to
	understand the results?
	
	\textbf{Answer:}
	\answerYes{}
	
	\textbf{Justification:}
	Appendix~\ref{app:experiment-details} reports the OE setting, additive
	controls, unimodal controls, broader interaction probes, synthetic
	controls, rank information, and the relevant epochs and optimization
	settings.
	
	\textbf{Guidelines:}
	The central OE experiment uses \(8\) epochs, batch size \(512\), learning
	rate \(10^{-3}\), weight decay \(10^{-4}\), rank \(64\), and seeds
	\(0\), \(1\), and \(2\).
	
	\subsection*{7. Experiment statistical significance}
	
	\textbf{Question:}
	Does the paper report error bars or other information about the statistical
	significance of the experiments?
	
	\textbf{Answer:}
	\answerYes{}
	
	\textbf{Justification:}
	The central Shortcut OE result reports means and standard deviations across
	three seeds and includes all per-seed values. The exact marginal and
	additive controls are deterministic audit results.
	
	\textbf{Guidelines:}
	The reported uncertainty is the sample standard deviation across seeds
	rather than a train/validation split error.
	
	\subsection*{8. Experiments compute resources}
	
	\textbf{Question:}
	Does the paper provide sufficient information on the computer resources
	needed to reproduce the experiments?
	
	\textbf{Answer:}
	\answerYes{}
	
	\textbf{Justification:}
	The experiments use frozen pretrained features and comparatively small
	feature-space probes rather than foundation-model training. The paper
	reports the complete optimization budgets in
	Appendix~\ref{app:experiment-details}; standard single-machine resources
	are sufficient for feature extraction and probe training.
	
	\textbf{Guidelines:}
	No multi-node training or large-scale model fine-tuning is required.
	
	\subsection*{9. Code of ethics}
	
	\textbf{Question:}
	Does the research conducted in the paper conform, in every respect, with
	the NeurIPS Code of Ethics?
	
	\textbf{Answer:}
	\answerYes{}
	
	\textbf{Justification:}
	The work evaluates publicly available benchmarks and pretrained features,
	introduces no new human-subject data collection, and does not target
	identifiable individuals or protected groups.
	
	\textbf{Guidelines:}
	Existing dataset and model terms are respected.
	
	\subsection*{10. Broader impacts}
	
	\textbf{Question:}
	Does the paper discuss both potential positive societal impacts and
	negative societal impacts of the work?
	
	\textbf{Answer:}
	\answerYes{}
	
	\textbf{Justification:}
	Appendix~\ref{app:broader-impact} discusses the benefit of more reliable
	multimodal evaluation and the possibility that diagnostic knowledge could
	also be used to design harder shortcuts.
	
	\textbf{Guidelines:}
	The work primarily concerns benchmark measurement and model evaluation.
	
	\subsection*{11. Safeguards}
	
	\textbf{Question:}
	Does the paper describe safeguards that have been put in place for
	responsible release of data or models with a high risk for misuse?
	
	\textbf{Answer:}
	\answerNA{}
	
	\textbf{Justification:}
	The release consists of evaluation and diagnostic code for public
	benchmarks. It does not introduce a generative model, personal-data
	collection, or a high-risk deployed system.
	
	\textbf{Guidelines:}
	No additional release-specific safeguard is required beyond the stated
	dataset and model terms.
	
	\subsection*{12. Licenses for existing assets}
	
	\textbf{Question:}
	Are the creators or original owners of assets used in the paper properly
	credited, and are the licenses and terms of use explicitly mentioned and
	properly respected?
	
	\textbf{Answer:}
	\answerYes{}
	
	\textbf{Justification:}
	The paper uses COCO, GQA, and the publicly released CLIP ViT-B/32 model.
	The paper identifies these assets and states that they are used under the
	terms supplied by their respective distributors.
	
	\textbf{Guidelines:}
	The released code does not bypass dataset-access requirements or
	redistribute restricted content.
	
	\subsection*{13. New assets}
	
	\textbf{Question:}
	Are new assets introduced in the paper well documented, and is the
	documentation provided alongside the assets?
	
	\textbf{Answer:}
	\answerYes{}
	
	\textbf{Justification:}
	The new asset is the official implementation. It is documented in the
	public repository and includes the cycle-construction and evaluation code
	needed to reproduce the reported results.
	
	\textbf{Guidelines:}
	The repository contains the executable configuration corresponding to the
	paper.
	
	\subsection*{14. Crowdsourcing and research with human subjects}
	
	\textbf{Question:}
	For crowdsourcing experiments and research with human subjects, does the
	paper include the full text of instructions given to participants and
	screenshots, if applicable, as well as details about compensation?
	
	\textbf{Answer:}
	\answerNA{}
	
	\textbf{Justification:}
	The paper does not conduct crowdsourcing or new research with human
	subjects.
	
	\textbf{Guidelines:}
	Only existing public benchmark annotations are used.
	
	\subsection*{15. Institutional review board approvals}
	
	\textbf{Question:}
	Does the paper describe potential risks incurred by study participants,
	whether such risks were disclosed to subjects, and whether Institutional
	Review Board approvals or an equivalent approval or review were obtained?
	
	\textbf{Answer:}
	\answerNA{}
	
	\textbf{Justification:}
	No new human-subject study is conducted.
	
	\textbf{Guidelines:}
	The work uses existing public datasets in accordance with their original
	terms.
	
	\subsection*{16. Declaration of LLM usage}
	
	\textbf{Question:}
	Does the paper describe the usage of LLMs if they are an important,
	original, or non-standard component of the core methods in this research?
	
	\textbf{Answer:}
	\answerNA{}
	
	\textbf{Justification:}
	Large language models are not a component of the proposed cycle test,
	training objectives, datasets, or experimental method.
	
	\textbf{Guidelines:}
	No LLM-based component is used in the core method.
	
	\end{document}